\documentclass[pdflatex,sn-mathphys-num]{sn-jnl}

\usepackage{graphicx}%
\usepackage{multirow}%
\usepackage{amsmath,amssymb,amsthm,mathtools}
\usepackage{mathrsfs}
\usepackage[title]{appendix}%
\usepackage{xcolor}%
\usepackage{textcomp}%
\usepackage{manyfoot}%
\usepackage{booktabs}%
\usepackage{algorithm}%
\usepackage{algorithmicx}%
\usepackage{algpseudocode}%
\usepackage{listings}%
\usepackage{enumitem}

\usepackage[nameinlink,capitalize]{cleveref}

\newcommand{\cU}{\mathcal{U}}
\newcommand{\cV}{\mathcal{V}}
\newcommand{\cX}{\mathcal{X}}

\newcommand{\cW}{\mathcal{W}}

\newcommand{\cD}{\mathcal{D}}

\theoremstyle{thmstyleone}%
\newtheorem{theorem}{Theorem}%
\newtheorem{lemma}[theorem]{Lemma}

\theoremstyle{definition}
\newtheorem{definition}[theorem]{Definition}

\theoremstyle{remark}
\newtheorem{remark}[theorem]{Remark}
\graphicspath{{Figs/}}

\newcommand{\SL}[1]{\textcolor{black}{#1}}

\begin{document}

\title[Landscapes of Low-Tubal-Rank Tensor Sensing]{Strict-Saddle Landscapes and Multi-Rank Geometry in Low-Tubal-Rank Tensor Sensing}

\author[1]{\fnm{Eugene} \sur{Agyei-Kodie}  }\email{agyeikod@msu.edu}\equalcont{These authors contributed equally to this work.}

\author*[2,1]{\fnm{Longxiu} \sur{Huang}}\email{huangl3@msu.edu}
\equalcont{These authors contributed equally to this work.}

\author[3]{\fnm{Shuang} \sur{Li}}\email{lishuang@iastate.edu}
\equalcont{These authors contributed equally to this work.}

\author[3]{\fnm{Xiao} \sur{Liang}}\email{liangx@iastate.edu}
\equalcont{These authors contributed equally to this work.}

\affil[1]{\orgdiv{Department of Mathematics}, \orgname{Michigan State University},   \orgaddress{\street{619 Red Cedar}, \city{ East Lansing}, \postcode{48824}, \state{MI}, \country{USA}}}

\affil[2]{\orgdiv{Department of Computational Mathematics, Science and Engineering}, \orgname{Michigan State University},   \orgaddress{\street{428 S Shaw Ln}, \city{ East Lansing}, \postcode{48824}, \state{MI}, \country{USA}}}

\affil[3]{\orgdiv{Department of Electrical and Computer Engineering}, \orgname{Iowa State University}, \orgaddress{\street{2520 Osborn Dr}, \city{Ames}, \postcode{50011}, \state{IA}, \country{USA}}}


\abstract{
We study the optimization landscape of low-tubal-rank tensor sensing through a balanced factorization. Under a tubal restricted isometry condition, we establish a quantitative strict-saddle landscape with no spurious local minima for arbitrary Fourier multi-rank profiles. We further show that the local geometry depends on the Fourier-slice ranks rather than the tubal rank alone. Uniform ranks yield quadratic growth transverse to the solution orbit, whereas nonuniform ranks produce quartically flat directions through hidden frequency-wise overparameterization, even when the factor width equals the exact tubal rank. Numerical experiments illustrate the global optimization behavior and the contrasting local geometries.
}




\maketitle
 
\section{Introduction}
\label{sec:introduction}

Low-rank models provide an effective way to exploit hidden structure in
high-dimensional data and have become fundamental tools in signal processing,
imaging, machine learning, and scientific computing.  While matrix models are
well suited to two-dimensional data, many applications naturally produce
multidimensional arrays, including color images, hyperspectral
images, seismic data, and biomedical imaging data
\cite{kolda2009tensor,liu2013tensor,fan2017hyperspectral,kreimer2012tensor}.
Tensor representations preserve this multidimensional structure and have
motivated a broad range of recovery, completion, and approximation methods.

Unlike matrices, tensors admit several inequivalent notions of rank.  In this
paper, we consider the tubal rank associated with the t-product
\cite{kilmer2011factorization,kilmer2013thirdorder}.  The t-product provides a
matrix-like algebra for third-order tensors: after applying the discrete
Fourier transform along the third mode, tensor multiplication becomes ordinary
matrix multiplication between corresponding frontal slices.  This leads to
the tensor singular value decomposition (t-SVD) and makes low-tubal-rank
factorization particularly attractive for large-scale recovery problems.

Let
\(\mathcal X_\star\in\mathbb R^{n_1\times n_2\times n_3}\) be a tensor of
tubal rank \(r\), and suppose that we observe
\(y=\mathcal M(\mathcal X_\star)\), where
\(\mathcal M:\mathbb R^{n_1\times n_2\times n_3}\to\mathbb R^m\)
is a linear sensing operator.  A natural nonconvex approach is to factorize
the unknown tensor as \(\mathcal X=\mathcal U*\mathcal V^\ast\),\footnote{Throughout, \( * \) denotes
the t-product when used as a binary operator, a superscript \(^{*}\) denotes the t-adjoint.}  with
\(\mathcal U\in\mathbb R^{n_1\times r\times n_3}\) and
\(\mathcal V\in\mathbb R^{n_2\times r\times n_3}\), and optimize directly over
the factors.  We study the balanced objective
\begin{equation*}
f(\mathcal U,\mathcal V)
:=
\frac12
\left\|
\mathcal M
\bigl(
\mathcal U*\mathcal V^\ast-\mathcal X_\star
\bigr)
\right\|_2^2
+
\frac18
\left\|
\mathcal U^\ast*\mathcal U-\mathcal V^\ast*\mathcal V
\right\|_F^2 .
\end{equation*}
The balancing term controls the noncompact scaling ambiguity between the two
factors while preserving their common t-orthogonal symmetry.  If
\(\mathcal X_\star=\mathcal U_\star*\mathcal V_\star^\ast\) is a balanced
factorization of the ground truth, then the corresponding solution orbit is
\(\mathcal S_\star
=\{(\mathcal U_\star*\mathcal Q,\mathcal V_\star*\mathcal Q):
\mathcal Q\in\mathsf O_{\mathrm t}(r)\}\), where \(\mathsf O_{\mathrm t}(r)\) denotes the group of \(r\times r\times n_3\) real t-orthogonal tensors.

For low-rank matrix sensing, factorized formulations are known to have a
favorable optimization landscape under suitable restricted isometry conditions: local minima are globally optimal, while nonglobal critical points
possess directions of negative curvature
\cite{bhojanapalli2016global,ge2017no}.  Procrustes-type estimates further
relate the distance between matrix factors to the error in the corresponding
lifted low-rank matrices and play an important role in both landscape and
convergence analyses \cite{tu2016low}.

At first sight, one might expect a similar picture for low-tubal-rank tensor
sensing.  Indeed, after applying the discrete Fourier transform along the
third mode, the t-product becomes matrix multiplication between corresponding
Fourier slices.  A general sensing operator \(\mathcal M\), however, need not
act independently on these slices, so the resulting objective does not in
general decompose into a collection of independent matrix-sensing problems.
It is therefore not immediate that the benign landscape known in the matrix
setting carries over to the tensor problem.

\vspace{0.1in}
\noindent{\bf A motivating example.}
Figure~\ref{fig:asymmetric-tubal-sensing} gives a simple numerical indication
of the global optimization behavior.  For a fixed low-tubal-rank target and a
fixed Gaussian sensing operator, we minimize the balanced objective from
100 independent random initializations.  Every run reaches essentially exact
recovery, with final relative errors around machine precision and objective
values close to zero.  The result suggests that the factorized tensor-sensing objective
may have a benign geometry beyond a neighborhood of a specially constructed
initialization.

\begin{figure*}[htb!]
    \centering
    \includegraphics[width=0.94\textwidth]
    {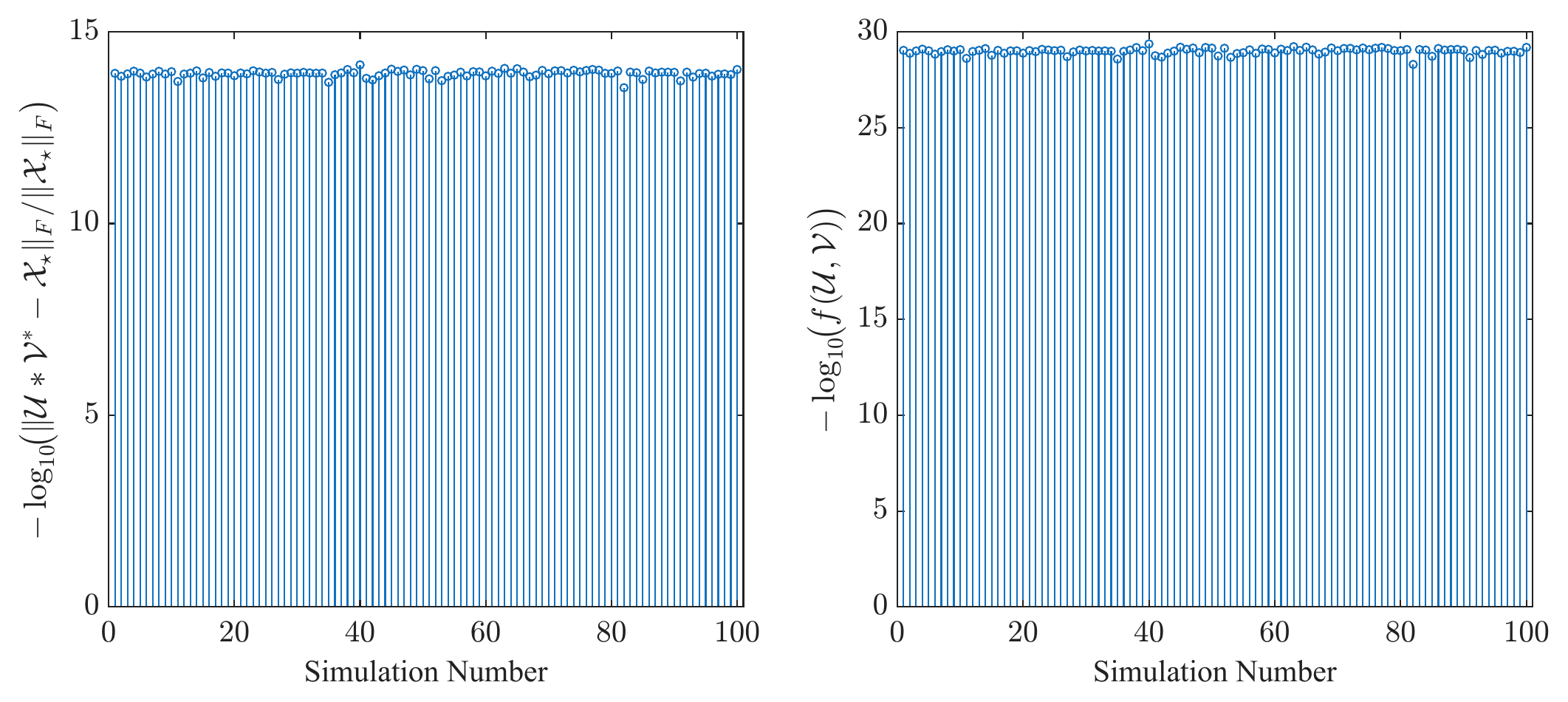}
    \caption{
    Consistent near-exact recovery across 100 independent random
    initializations.  The target is generated from two independent standard
    Gaussian factors of size \(6\times1\times3\) and normalized to unit
    Frobenius norm.  We use \(m=70\) noiseless Gaussian measurements with
    sensing entries distributed as \(\mathcal N(0,1/m)\).  The target and
    sensing operator are fixed across all runs, while the initial factors are
    drawn independently from \(\mathcal N(0,0.1^2)\).  Each run performs
    5000 gradient-descent iterations with stepsize \(0.05\).  The left panel
    reports \(-\log_{10}\) of the final relative reconstruction error, and the
    right panel reports \(-\log_{10}\) of the final balanced objective value.
    }
    \label{fig:asymmetric-tubal-sensing}
\end{figure*}

The behavior in Figure~\ref{fig:asymmetric-tubal-sensing} motivates the first
question studied in this paper: {\em under what conditions does the favorable
global landscape of low-rank matrix sensing persist for low-tubal-rank tensor
sensing?}

The tensor setting contains, however, an additional structural feature that is
not visible in this experiment.  Let \(\widehat{\mathcal X}_\star\) denote the discrete Fourier transform of \(\mathcal X_\star\) along the third mode, and \(\widehat{\mathcal X}_\star^{(k)}\) denote its \(k\)-th frontal slice.  For each Fourier frequency, let
\(\rho_k=\operatorname{rank}(\widehat{\mathcal X}_\star^{(k)})\).  The vector
\(\boldsymbol\rho_\star=(\rho_1,\ldots,\rho_{n_3})\) is the multi-rank of the
ground-truth tensor, whereas its tubal rank is only
\(r=\max_k\rho_k\).  Thus the tubal rank records the largest Fourier-slice rank
but does not determine the complete Fourier rank profile. 
This distinction has an important consequence for factorized optimization. Choosing the factor width equal to the exact tubal rank \(r\) does not imply
that the factorization is exactly parameterized at every Fourier frequency.
Whenever \(\rho_k<r\), the width-\(r\) factors retain inactive coordinates at
that frequency.  Hence overparameterization may arise even though the chosen
factor width matches the true tubal rank.

This leads to a second question that has no direct analogue in exactly
parameterized matrix sensing: {\em if rank deficiency at some Fourier frequencies
does not destroy the favorable global landscape, does it nevertheless change
the local geometry near the solution set?}

Our results show that the global and local aspects of the problem behave
differently.  Under a suitable tensor restricted isometry condition, the
no-spurious-minima property and the strict-saddle structure persist for
arbitrary multi-rank profiles; in particular, uniform Fourier ranks are not
required.  The local geometry, however, is governed by the full multi-rank
profile.  When every Fourier slice has rank equal to the tubal rank, the
factorization exhibits the nondegenerate quadratic behavior familiar from
exactly parameterized matrix models.  When some Fourier slice has smaller
rank, inactive factor coordinates generate additional flat directions beyond
those associated with the t-orthogonal symmetry.

This separation clarifies the roles played by tubal rank and multi-rank.  The
tubal rank determines the factor width, while the multi-rank determines
whether that factorization is locally nondegenerate.  Consequently, tensors
with the same tubal rank can exhibit qualitatively different local
optimization landscapes because of differences in their Fourier rank
profiles.

\subsection{Related work}
\label{subsec:related-work}

\noindent{}{\bf Low-tubal-rank tensor recovery.} 
The t-product and t-SVD framework
\cite{kilmer2011factorization,kilmer2013thirdorder} has been widely used in
low-rank tensor recovery.  Convex formulations based on tensor nuclear norms
have been developed for tensor completion, robust tensor recovery, and related
inverse problems \cite{zhang2017exact,lu2020tensor,su2024guaranteed}.  For
general linear measurements, the work \cite{zhang2021tensor} establishes
t-RIP for standard random sensing ensembles, providing a tensor analogue of the restricted isometry framework used in low-rank matrix sensing.

\vspace{0.1in}
\noindent{\bf Factorized low-tubal-rank recovery.}
Factorized formulations reduce the computational cost of low-tubal-rank recovery by optimizing over smaller tensor factors and avoiding repeated full t-SVD computations.  The work \cite{assoweh2022low} studies a
regularized Burer--Monteiro formulation for tensor completion under random tube-wise sampling, where selecting an index pair \((i,j)\) reveals the entire
mode-three tube \(\mathcal X(i,j,:)\).  Since this sampling operator commutes with the Fourier transform along the third mode, the data-fitting term separates across Fourier frequencies.  Together with the frequency-wise regularizers, this structure permits a slice-wise matrix-completion analysis with a common sampling mask.  The paper presents no-spurious-local-minima and quantitative strict-saddle results under suitable incoherence and sampling conditions.  Its quantitative strict-saddle argument relies on a positive smallest-singular-value coercivity bound across all frequencies.  Such a bound implicitly requires the width-\(r\) ground-truth factors to have full column rank at every Fourier frequency, or equivalently, \(\rho_k=r\) for all \(k\).

Recent work also investigates the convergence and statistical accuracy of factorized methods under linear measurements.  Factorized gradient descent in \cite{liu2024lowtubal} is analyzed in both noiseless and noisy settings, including cases in which the factor width overestimates the true tubal rank.
Implicit regularization is studied in \cite{karnik2025implicit}, which shows that gradient descent with small random initialization in an overparameterized tubal factorization can favor low-tubal-rank solutions.  For noisy recovery with a symmetric factorization, \cite{liu2026smallinit} shows that small initialization, combined with an appropriate stopping rule, yields recovery
error bounds governed by the true tubal rank rather than the overestimated factor width.  \SL{Liu et al.~\cite{liu2025apgd} also identify frequency-wise overparameterization even when the factor width equals the true tubal rank, and use the multi-rank profile to establish linear convergence guarantees for alternating preconditioned gradient descent.}  These algorithmic results describe the behavior of iterates under specified initialization and measurement conditions.

The present work complements these contributions by studying the balanced asymmetric sensing objective over the entire factor space under t-RIP, without requiring the sensing loss to separate across Fourier frequencies. Our analysis also distinguishes the roles of tubal rank and multi-rank in the local landscape.  Even when the factor width equals the exact tubal rank \(r\), a Fourier slice with rank \(\rho_k<r\) remains overparameterized. We show that this frequency-wise rank deficiency creates normal Hessian-kernel directions beyond the t-orthogonal symmetry, with quartic objective growth and cubic gradient growth along suitable perturbations. By contrast, when \(\rho_k=r\) at every frequency, the Hessian kernel consists exactly of the symmetry directions and the objective has quadratic transverse growth.  This distinction explains how a benign global landscape can coexist with local degeneracy and why the quantitative strict-saddle geometry depends on the Fourier rank profile.

\vspace{0.1in}
\noindent{\bf Low-rank matrix landscapes and overparameterization.} For low-rank matrix sensing, the works \cite{bhojanapalli2016global,ge2017no} show that suitable RIP conditions lead to a benign factorized landscape with no spurious local minima and strict saddles away from the global solution set.  The Procrustes geometry developed in \cite{tu2016low} provides quantitative relations between factor-space error and lifted matrix error and plays a central role in related landscape and convergence analyses.

Overparameterized matrix factorizations can exhibit additional degeneracy and slower local behavior; see, for example,
\cite{zhuo2024overparameterized}.  \SL{For overparameterized positive-semidefinite matrix factorization, Davis et al.~\cite{davis2025fourthorder} characterize quartic objective growth on a suitable manifold and give a cubic upper bound on the full gradient norm on that manifold.}  The tubal setting introduces a distinct form of overparameterization: it can occur even when the chosen factor width equals the exact tubal rank.  Different Fourier frequencies may therefore be exactly parameterized and overparameterized simultaneously.  This frequency-dependent structure is what separates the two local geometric regimes identified in this paper.

\subsection{Main contributions}
\label{subsec:introduction-contributions}

Our main contributions are summarized as follows.

\begin{enumerate}[label=(\roman*),leftmargin=2em]

\item
Under a tubal restricted isometry condition of order \(2r\) with distortion below \(1/5\), we prove that every local minimum of the balanced objective is globally optimal and every nonglobal critical point is a strict saddle. The global minimizers are exactly the balanced ground-truth factorizations up to a common t-orthogonal transformation. We also establish a quantitative strict-saddle description over the entire factor space. These results hold for arbitrary Fourier multi-rank profiles.

\item
We establish a multi-rank dichotomy for the local geometry. When \(\rho_k=r\) at every frequency, the Hessian kernel at each global minimizer consists exactly of the symmetry-generated tangent directions, and the objective has quadratic growth transverse to the solution orbit. If some \(\rho_k<r\), frequency-wise overparameterization occurs even though the factor width equals the exact tubal rank. At every global minimizer, we construct additional Hessian-kernel directions normal to the solution orbit along which the objective grows quartically and the gradient norm grows cubically. Consequently, local quadratic growth and a local Polyak--\L{}ojasiewicz inequality fail.

\item For fixed problem data and gradient tolerance \(\epsilon>0\), we obtain strict-saddle proximity radii \(O(\epsilon)\) for uniform multi-rank profiles and \(O(\epsilon^{1/3})\) for arbitrary profiles. The latter estimate includes an explicit \(\epsilon\)-dependent negative-curvature threshold. In the nonuniform case, the quartically flat directions rule out any proximity radius \(o(\epsilon^{1/3})\) when the negative-curvature threshold is fixed and positive, explaining the obstruction to matrix-like linear proximity estimates.

\item
Numerical experiments illustrate recovery from independent random initializations and the transition from quadratic to quartic local growth for tensors with the same tubal rank but different multi-rank profiles. Experiments with targets derived from optical coherence tomography data further examine recovery and the effects of model mismatch.

\end{enumerate}
\subsection{Organization}
\label{subsec:introduction-organization}

The remainder of the paper is organized as follows.
Section~\ref{sec:preliminaries} introduces the t-product notation, the
low-tubal-rank sensing model, the tensor restricted isometry property, and the
product-space geometry used in the analysis.
Section~\ref{sec:main-results} first establishes the global strict-saddle
landscape and then develops the multi-rank-dependent local geometry.
Section~\ref{sec:numerics} presents numerical experiments illustrating the
global landscape, the multi-rank-dependent local behavior, and the OCT-based
sensing experiments.
Section~\ref{sec:conclusion} concludes the paper and discusses possible
extensions.


\section{Preliminaries and Problem Setup}
\label{sec:preliminaries}
\label{sec:problem-setup}

We first introduce the notation and basic t-product conventions used
throughout the paper, and then formulate the balanced low-tubal-rank
tensor sensing problem.

\subsection{Notation and the t-product}
\label{subsec:basic-notation}
\label{subsec:fourier-t-product}

For a positive integer \(n\), let
\(  [n]  := \{1,\ldots,n\}. \)
We use calligraphic letters for tensors and ordinary capital letters for matrices. The vector \(e_j\) denotes the \(j\)-th standard basis vector. For a complex matrix \(A\), \(A^H\) denotes its conjugate
transpose, \(\sigma_j(A)\) denotes its \(j\)-th largest singular value,
and \(A\succeq0\) means that \(A\) is Hermitian positive semidefinite.

For complex matrices of the same size, we use the real Frobenius inner product
\( \langle A,B\rangle_F  :=  \operatorname{Re}\operatorname{tr}(A^HB).
\)
For real tensors, \(\langle\cdot,\cdot\rangle_F\) and \(\|\cdot\|_F\) denote the usual Frobenius inner product and norm.

For a real tensor \(\mathcal A \in  \mathbb R^{n_1\times n_2\times n_3}, \) let \(\widehat{\mathcal A}\) denote its discrete Fourier transform along the third mode, and write \(  \widehat{\mathcal A}^{(k)}  \in  \mathbb C^{n_1\times n_2}\) for its \(k\)-th frontal slice. We use the unnormalized discrete Fourier transform, so Parseval's identity takes the form
\begin{equation}
\label{eq:parseval-tensor}
    \langle\mathcal A,\mathcal B\rangle_F
    = \frac1{n_3}\sum_{k=1}^{n_3}  \operatorname{Re}\operatorname{tr} \left(     \widehat{\mathcal A}^{(k)H}     \widehat{\mathcal B}^{(k)} \right),
    \qquad
    \|\mathcal A\|_F^2 =  \frac1{n_3}   \sum_{k=1}^{n_3}   \left\|      \widehat{\mathcal A}^{(k)}\right\|_F^2.
\end{equation}

Because the tensors considered in this paper are real, their Fourier
slices satisfy the usual conjugate-symmetry condition. Whenever
Fourier-domain factors, alignments, or perturbations are constructed,
the slices at paired frequencies are chosen as complex conjugates,
while the choices at self-conjugate frequencies are real. This
ensures that their inverse Fourier transforms are real.

For conformable tensors \(\mathcal A\) and \(\mathcal B\), their
t-product is characterized by
\(  \widehat{\mathcal A*\mathcal B}^{(k)}  =   \widehat{\mathcal A}^{(k)} \widehat{\mathcal B}^{(k)},   k\in[n_3].
\)
The t-adjoint is characterized by \( \widehat{\mathcal A^\ast}^{(k)} = \widehat{\mathcal A}^{(k)H}\). The tubal rank of \(\mathcal A\) is
\[
    \operatorname{rank}_{\mathrm t}(\mathcal A) := \max_{k\in[n_3]} \operatorname{rank}  \bigl(   \widehat{\mathcal A}^{(k)}  \bigr).
\]

The t-identity tensor
\(  \mathcal I  \in  \mathbb R^{r\times r\times n_3}
\)
is characterized by
\(  \widehat{\mathcal I}^{(k)}  =  I_r,    k\in[n_3].\)
The t-orthogonal group is
\(  \mathsf O_{\mathrm t}(r)  :=
    \left\{  \mathcal Q
        \in \mathbb R^{r\times r\times n_3}: \mathcal Q^\ast*\mathcal Q =  \mathcal Q*\mathcal Q^\ast  =   \mathcal I
    \right\}.
\)
Equivalently,
\begin{equation*}
    \mathcal Q\in\mathsf O_{\mathrm t}(r)
    \quad\Longleftrightarrow\quad
    \widehat{\mathcal Q}^{(k)}
    \text{ is unitary for every }k\in[n_3].
\end{equation*}

We repeatedly use the standard identities
\(  (\mathcal A*\mathcal B)^\ast  = \mathcal B^\ast*\mathcal A^\ast\)
and \(    \langle\mathcal A*\mathcal B,  \mathcal C \rangle_F  =  \langle  \mathcal B,  \mathcal A^\ast*\mathcal C\rangle_F = \langle     \mathcal A,   \mathcal C*\mathcal B^\ast \rangle_F.
\)
If \(\mathcal Q\) is t-orthogonal, then
\begin{equation*}
    \|\mathcal A*\mathcal Q\|_F
    =
    \|\mathcal A\|_F,
    \qquad
    \|\mathcal Q*\mathcal B\|_F
    =
    \|\mathcal B\|_F.
\end{equation*}
All of these identities follow by applying the corresponding matrix
identities to the Fourier slices and then using
\eqref{eq:parseval-tensor}.

\subsection{Low-tubal-rank tensor sensing}
\label{subsec:problem-setup}

Let \(    \mathcal X_\star    \in    \mathbb R^{n_1\times n_2\times n_3}\)
be an unknown low-tubal-rank tensor. For each \(k\in[n_3]\), define \(  \rho_k := \operatorname{rank}  \bigl(      \widehat{\mathcal X}_\star^{(k)}   \bigr). \)
The vector
\[\bf{\rho}_\star := (\rho_1,\ldots,\rho_{n_3})
\]
is called the \emph{multi-rank} of
\(\mathcal X_\star\). 
\begin{equation*}
    r := \operatorname{rank}_{\mathrm t}(\mathcal X_\star) = \max_{k\in[n_3]}\rho_k
\end{equation*}
is called the \textit{tubal-rank} of $\mathcal{X}_\star$ and we assume throughout that \(\mathcal X_\star \neq 0\), so that \(1 \le r \le \min\{n_1,n_2\}\). The degenerate zero-tensor case is omitted. We use a factorization of width \(r\). Although this factor width equals the exact tubal rank, it is overparameterized at every frequency for which \(\rho_k<r\). Using SVDs of the Fourier slices, chosen
compatibly with the conjugate-symmetry condition, we fix a balanced
factorization
\begin{equation*}
    \mathcal X_\star
    =
    \mathcal U_\star*\mathcal V_\star^\ast,
    \qquad
    \mathcal U_\star^\ast*\mathcal U_\star
    =
    \mathcal V_\star^\ast*\mathcal V_\star,
\end{equation*}
where \(   \mathcal U_\star  \in  \mathbb R^{n_1\times r\times n_3},  \mathcal V_\star   \in   \mathbb R^{n_2\times r\times n_3}. \)

Let
\[
    \mathcal M:
    \mathbb R^{n_1\times n_2\times n_3}
    \longrightarrow
    \mathbb R^m
\]
be a linear sensing operator. We consider the noiseless measurements
\begin{equation*}
    y
    =
    \mathcal M(\mathcal X_\star).
\end{equation*}
The adjoint
\(
    \mathcal M^\ast:
    \mathbb R^m
    \longrightarrow
    \mathbb R^{n_1\times n_2\times n_3}
\)
is characterized by
\begin{equation*}
    \langle
        \mathcal M(\mathcal Z),
        z
    \rangle_{\mathbb R^m}
    =
    \langle
        \mathcal Z,
        \mathcal M^\ast(z)
    \rangle_F
\end{equation*}
for every tensor \(\mathcal Z\) and every \(z\in\mathbb R^m\).

For candidate factors
\(
    \mathcal U
    \in
    \mathbb R^{n_1\times r\times n_3}
    ~\text{and}~
    \mathcal V
    \in
    \mathbb R^{n_2\times r\times n_3},
\)
define the reconstruction and balancing residuals by
\begin{equation*}
    \mathcal E(\mathcal U,\mathcal V)
    :=
    \mathcal U*\mathcal V^\ast-\mathcal X_\star,
    \qquad
    \mathcal B(\mathcal U,\mathcal V)
    :=
    \mathcal U^\ast*\mathcal U
    -
    \mathcal V^\ast*\mathcal V.
\end{equation*}
We study the balanced factorized objective
\begin{align}
    f(\mathcal U,\mathcal V)
    &:=   \frac12
    \left\|  \mathcal M   \bigl(
            \mathcal E(\mathcal U,\mathcal V)
        \bigr)
    \right\|_2^2  +  \frac18
    \left\|   \mathcal B(\mathcal U,\mathcal V)
    \right\|_F^2.
\label{eq:balanced-sensing-objective}
\end{align}

The balancing term controls the noncompact ambiguity between the two
factors while preserving their common t-orthogonal symmetry. In
particular, for every
\(\mathcal Q\in\mathsf O_{\mathrm t}(r)\),
\begin{equation*}
    f
    \bigl(
        \mathcal U*\mathcal Q,
        \mathcal V*\mathcal Q
    \bigr)
    =
    f(\mathcal U,\mathcal V).
\end{equation*}
We therefore define the balanced ground-truth factor orbit by
\begin{equation*}
    \mathcal S_\star
    :=
    \left\{
        \bigl(
            \mathcal U_\star*\mathcal Q,\,
            \mathcal V_\star*\mathcal Q
        \bigr):
        \mathcal Q\in\mathsf O_{\mathrm t}(r)
    \right\}.
\end{equation*}
Theorem~\ref{thm:global-strict-saddle} will show that this orbit is
exactly the set of global minimizers of
\eqref{eq:balanced-sensing-objective}.

\begin{definition}[t-RIP]
\label{def:t-RIP}
Let \(s\geq1\) and \(\delta\in(0,1)\). We say that
\(\mathcal M\) satisfies the \((s,\delta)\)-t-RIP if
\begin{equation*}
    (1-\delta)\|\mathcal Z\|_F^2
    \leq
    \|\mathcal M(\mathcal Z)\|_2^2
    \leq
    (1+\delta)\|\mathcal Z\|_F^2
\end{equation*}
for every tensor \(\mathcal Z\) satisfying
\(  \operatorname{rank}_{\mathrm t}(\mathcal Z)  \leq s.\)
\end{definition}

\begin{remark} 
\label{rem:random-ensembles-TRIP}
The t-RIP assumption holds with high probability for standard random
sensing operators. Suppose that
\[
    [\mathcal M(\mathcal Z)]_\ell
    =
    \langle
        \mathcal A_\ell,
        \mathcal Z
    \rangle_F,
    \qquad
    \ell=1,\ldots,m,
\]
where the entries of the sensing tensors \(\mathcal A_\ell\) are
independent, centered, variance-\(1/m\) sub-Gaussian random
variables. It is shown in \cite{zhang2021tensor} that, for
\(s\geq1\), \(\delta\in(0,1)\), and \(\eta\in(0,1)\), the operator
\(\mathcal M\) satisfies the \((s,\delta)\)-t-RIP with probability at
least \(1-\eta\), provided
\[
    m
    \geq
    C\delta^{-2}
    \max
    \left\{
        s(n_1+n_2+1)n_3,\,
        \log(\eta^{-1})
    \right\},
\]
where \(C>0\) depends only on the sub-Gaussian parameter. This class
includes normalized Gaussian and symmetric Bernoulli sensing
ensembles. In particular, taking \(s=2r\) verifies the t-RIP
assumption used in Theorem~\ref{thm:global-strict-saddle}.
\end{remark}

\subsection{Product-space geometry}
\label{subsec:product-factor-space}

We identify a factor pair with the vertically stacked tensor
\begin{equation*}
    \mathcal W
    :=
    \begin{bmatrix}
        \mathcal U\\
        \mathcal V
    \end{bmatrix}
    \in
    \mathbb R^{(n_1+n_2)\times r\times n_3},
    \qquad
    \mathcal W_\star
    :=
    \begin{bmatrix}
        \mathcal U_\star\\
        \mathcal V_\star
    \end{bmatrix}.
\end{equation*}
Likewise, a direction in the product factor space is written as
\[
    \mathcal Z
    :=
    \begin{bmatrix}
        \mathcal Z_{\mathcal U}\\
        \mathcal Z_{\mathcal V}
    \end{bmatrix}.
\]
We use the product Frobenius inner product
\(   \langle      \mathcal W,  \mathcal Z \rangle_F :=  \langle     \mathcal U,     \mathcal Z_{\mathcal U} \rangle_F +  \langle     \mathcal V,     \mathcal Z_{\mathcal V} \rangle_F,\)
and hence
\(
    \|\mathcal W\|_F^2
    =
    \|\mathcal U\|_F^2
    +
    \|\mathcal V\|_F^2.
\)
We also write
\(f(\mathcal W) :=f(\mathcal U,\mathcal V), \mathcal E(\mathcal W) :=\mathcal E(\mathcal U,\mathcal V),    \mathcal B(\mathcal W) := \mathcal B(\mathcal U,\mathcal V).
\)

The associated stacked Gram tensor is
\begin{equation*}
    \mathcal W*\mathcal W^\ast
    =
    \begin{bmatrix}
        \mathcal U*\mathcal U^\ast
        &
        \mathcal U*\mathcal V^\ast
        \\
        \mathcal V*\mathcal U^\ast
        &
        \mathcal V*\mathcal V^\ast
    \end{bmatrix}.
\end{equation*}

The distance from \(\mathcal W\) to the ground-truth factor orbit is defined as the \emph{Procrustes distance}:
\begin{equation*}
\begin{aligned}
    \operatorname{dist}
    \bigl(
        \mathcal W,
        \mathcal S_\star
    \bigr)
    &:=
    \min_{\mathcal Q\in\mathsf O_{\mathrm t}(r)}
    \left\|       \mathcal W-\mathcal W_\star*\mathcal Q
    \right\|_F= \min_{\mathcal Q\in\mathsf O_{\mathrm t}(r)}
    \left(    \|\mathcal U-\mathcal U_\star*\mathcal Q\|_F^2+ \|\mathcal V-\mathcal V_\star*\mathcal Q\|_F^2
    \right)^{1/2}.
\end{aligned}
\end{equation*}

All differentiation is performed in the original real product factor
space. We use the conventions
\begin{equation*}
    Df(\mathcal W)[\mathcal Z]
    =
    \langle
        \nabla f(\mathcal W),
        \mathcal Z
    \rangle_F,
    \qquad
    D^2f(\mathcal W)[\mathcal Z,\mathcal Z]
    =
    \left\langle
        \mathcal Z,
        \nabla^2f(\mathcal W)\mathcal Z
    \right\rangle_F.
\end{equation*}
The smallest Hessian eigenvalue is
\begin{equation*}
    \lambda_{\min}
    \bigl(
        \nabla^2f(\mathcal W)
    \bigr)
    :=
    \min_{\|\mathcal Z\|_F=1}
    D^2f(\mathcal W)
    [\mathcal Z,\mathcal Z].
\end{equation*}
A critical point \(\mathcal W\) is called a strict saddle if
\[
    \lambda_{\min}
    \bigl(
        \nabla^2f(\mathcal W)
    \bigr)
    <0.
\]

\begin{definition} 
\label{def:quantitative-strict-saddle}
Let \(\epsilon,\gamma,\zeta>0\). We say that \(f\) is
\((\epsilon,\gamma,\zeta)\)-strict saddle relative to
\(\mathcal S_\star\) if every factor pair \(\mathcal W\) satisfies
at least one of
\[
    \|\nabla f(\mathcal W)\|_F
    \geq
    \epsilon,
    \qquad
    \lambda_{\min}
    \bigl(
        \nabla^2f(\mathcal W)
    \bigr)
    \leq
    -\gamma,
    \qquad
    \operatorname{dist}
    \bigl(
        \mathcal W,
        \mathcal S_\star
    \bigr)
    \leq
    \zeta.
\]
\end{definition}
\paragraph{Tangent and normal spaces.}
\label{subsec:tangent-normal-spaces}
For
\[
    \overline{\mathcal W}
    =
    \begin{bmatrix}
        \overline{\mathcal U}\\
        \overline{\mathcal V}
    \end{bmatrix}
    \in
    \mathcal S_\star,
\]
the tangent space
\(T_{\overline{\mathcal W}}\mathcal S_\star\) consists of the
derivatives at \(t=0\) of differentiable curves in
\(\mathcal S_\star\) passing through \(\overline{\mathcal W}\).
Since \(\mathcal S_\star\) is the orbit of the common
t-orthogonal action,
\begin{equation*}
    T_{\overline{\mathcal W}}\mathcal S_\star
    =
    \left\{
        \begin{bmatrix}
            \overline{\mathcal U}*\mathcal K\\
            \overline{\mathcal V}*\mathcal K
        \end{bmatrix}
        :
        \mathcal K
        \in
        \mathbb R^{r\times r\times n_3},
        \quad
        \mathcal K^\ast=-\mathcal K
    \right\}.
\end{equation*}
Indeed, the skew-adjoint condition follows by differentiating
\[
    \mathcal Q(t)^\ast*\mathcal Q(t)
    =
    \mathcal I,
\]
and the reverse inclusion follows by taking the frequency-wise matrix
exponential of a skew-adjoint tensor.

The normal space is the orthogonal complement
\begin{equation*}
    N_{\overline{\mathcal W}}\mathcal S_\star
    :=
    \left(
        T_{\overline{\mathcal W}}\mathcal S_\star
    \right)^\perp
\end{equation*}
with respect to the product Frobenius inner product.

 
\section{Main Results}
\label{sec:main-results}

\subsection{Global Strict-Saddle Landscape under t-RIP}
\label{subsec:global-strict-saddle}
We first establish a quantitative strict-saddle description over the entire factor space. The result holds for arbitrary multi-ranks, including the nonuniform case in which some Fourier slices have rank strictly smaller than the tubal rank \(r\). 

\begin{theorem} 
\label{thm:global-strict-saddle}
Let \(\mathcal X_\star\in\mathbb R^{n_1\times n_2\times n_3}\) be nonzero and have tubal rank \(r \ge 1\), and consider
\[   f(\cU,\cV)   =    \frac12   \left\|   \mathcal M    \bigl(  \cU*\cV^\ast-\mathcal X_\star   \bigr)  \right\|_2^2   +   \frac18   \left\|     \cU^\ast*\cU  -   \cV^\ast*\cV   \right\|_F^2.\]
Assume that \(\mathcal M\) satisfies the \((2r,\delta)\)-t-RIP with \(\delta<\frac1{5}.\) Then the following statements hold.

\begin{enumerate}[label=(\roman*),leftmargin=2em]

    \item The set of global minimizers is exactly
    \[     \mathcal S_\star     =    \left\{      \bigl(     \cU_\star*\mathcal Q,    \cV_\star*\mathcal Q  \bigr):   \mathcal Q\in\mathsf O_{\mathrm t}(r)    \right\}.  \]

    \item Let \(\mathcal{W}=(\cU,\cV)\), and set
    \( d(\cW)  :=    \operatorname{dist}(\cW,\mathcal S_\star).  \)
    If \(\cW\notin\mathcal S_\star\) is a critical point, then
    \begin{equation*}  
        \lambda_{\min}\bigl(\nabla^2 f(\cW)\bigr) \leq -\frac{1-5\delta}{8r}d(\cW)^2 < 0.   \end{equation*}
    Consequently, every nonglobal critical point is a strict saddle and every local minimum is globally optimal.

    \item  For any \(\epsilon,\gamma,\zeta>0\) satisfying
\(\gamma+\frac{4\epsilon}{\zeta}\leq\frac{1-5\delta}{8r}\,\zeta^2,\)
the objective \(f\) is
\((\epsilon,\gamma,\zeta)\)-strict saddle relative to
\(\mathcal S_\star\). That is, every factor pair
\(\mathcal W\) satisfies at least one of
\[
\|\nabla f(\mathcal W)\|_F\geq\epsilon,
\qquad
\lambda_{\min}\bigl(\nabla^2f(\mathcal W)\bigr)
\leq-\gamma,
\qquad
d(\mathcal W)\leq\zeta.
\]
In particular, for every \(\epsilon>0\), one may take
\[
\zeta=\left(\frac{64r\epsilon}{1-5\delta}\right)^{1/3},
\qquad
\gamma=\left(\frac{(1-5\delta)\epsilon^2}{r}\right)^{1/3}.
\]

\end{enumerate}
\end{theorem}

The proof is based on a symmetric lifting of the two-factor problem. We first establish the Procrustes and Hessian estimates, then use negative curvature outside \(\mathcal S_\star\) to identify all local and global minimizers.

\subsubsection{Tubal Procrustes alignment}
\label{subsubsec:tubal-Procrustes}

Choose
\begin{equation}
\label{eq:stacked-Procrustes}
    \mathcal Q_{\mathrm{opt}}     \in  \arg\min_{\mathcal Q\in\mathsf O_{\mathrm t}(r)}  \left\|    \cW-\cW_\star*\mathcal Q   \right\|_F.
\end{equation}
For each \(k\in[n_3]\), let \(\widehat{\cW}_\star^{(k)H} \widehat{\cW}^{(k)} =  L_k\Sigma_kR_k^H \) be a full SVD. By the  orthogonal Procrustes theorem \cite{schonemann1966generalized}, the minimizer in \eqref{eq:stacked-Procrustes} can be chosen so that \( \widehat{\mathcal Q}_{\mathrm{opt}}^{(k)}   =  L_kR_k^H.\) The SVDs can be chosen conjugate-symmetrically, so the resulting \(\mathcal Q_{\mathrm{opt}}\) is real and belongs to \(\mathsf O_{\mathrm t}(r)\). Moreover,
\begin{align}
    \left(
        \widehat{\cW}_\star^{(k)}
        \widehat{\mathcal Q}_{\mathrm{opt}}^{(k)}
    \right)^H
    \widehat{\cW}^{(k)}
 =
    \widehat{\mathcal Q}_{\mathrm{opt}}^{(k)H}
    \widehat{\cW}_\star^{(k)H}
    \widehat{\cW}^{(k)}
    =
    R_k\Sigma_kR_k^H
    \succeq0.
\label{eq:aligned-cross-Gram-PSD}
\end{align} 
Define the aligned error by
\(\cD:= \cW-\cW_\star*\mathcal Q_{\mathrm{opt}}.\)
Then \(  \|\cD\|_F   =  \operatorname{dist}\bigl(   \cW,\mathcal S_\star   \bigr).\)

\begin{lemma} 
\label{lem:aligned-factor-Gram}
The aligned error satisfies
\begin{equation}
\label{eq:aligned-factor-Gram}
    \frac1r
    \|\cD\|_F^4
    \leq
    \|\cD*\cD^\ast\|_F^2
    \leq
    2
    \left\|
        \cW*\cW^\ast
        -
        \cW_\star*\cW_\star^\ast
    \right\|_F^2.
\end{equation}
Consequently,
\begin{equation}
\label{eq:Gram-error-distance-lower0}
    \left\|
        \cW*\cW^\ast
        -
        \cW_\star*\cW_\star^\ast
    \right\|_F^2
    \geq
    \frac1{2r}
    \operatorname{dist}
    \bigl(
        \cW,\mathcal S_\star
    \bigr)^4.
\end{equation}
\end{lemma}

\begin{proof}
Fix \(k\in[n_3]\), and set
\(  W_k  :=\widehat{\cW}^{(k)},     A_k:= \widehat{\cW}_\star^{(k)}   \widehat{\mathcal Q}_{\mathrm{opt}}^{(k)},   D_k  :=  W_k-A_k.
\)
By \eqref{eq:aligned-cross-Gram-PSD}, we have \(  A_k^HW_k\succeq0. \)
In particular, \(A_k^HW_k\) is Hermitian and hence
\(  W_k^HA_k=A_k^HW_k.\)

Since \(D_kD_k^H\) and \(D_k^HD_k\) have the same nonzero eigenvalues,
\begin{align}
    \|D_kD_k^H\|_F^2  =   \|D_k^HD_k\|_F^2   =    \left\|
        W_k^HW_k+A_k^HA_k-2A_k^HW_k  \right\|_F^2.
\label{eq:aligned-error-direct-expansion}
\end{align}
Also,
\begin{align}
    \|W_kW_k^H-A_kA_k^H\|_F^2
    &=
    \|W_k^HW_k\|_F^2
    +
    \|A_k^HA_k\|_F^2
    -
    2\|A_k^HW_k\|_F^2.
\label{eq:Gram-error-direct-expansion}
\end{align}
Expanding \eqref{eq:aligned-error-direct-expansion} and comparing it with \eqref{eq:Gram-error-direct-expansion} gives
\begin{align*}
     2\|W_kW_k^H-A_kA_k^H\|_F^2
    -
    \|D_kD_k^H\|_F^2
    &=
    \|W_k^HW_k-A_k^HA_k\|_F^2
    +
    4
    \left\langle
        D_k^HD_k,
        A_k^HW_k
    \right\rangle
    \notag\\
    &=
    \|W_k^HW_k-A_k^HA_k\|_F^2
    +
    4
    \left\|
        D_k
        (A_k^HW_k)^{1/2}
    \right\|_F^2
    \geq0.
\end{align*}
Therefore,
\begin{equation}
\label{eq:aligned-matrix-upper-bound}
    \|D_kD_k^H\|_F^2
    \leq
    2\|W_kW_k^H-A_kA_k^H\|_F^2.
\end{equation} 
Since \(\widehat{\mathcal Q}_{\mathrm{opt}}^{(k)}\) is unitary, \(  A_kA_k^H=\widehat{\cW}_\star^{(k)}  \widehat{\cW}_\star^{(k)H}.
\)
Hence \eqref{eq:aligned-matrix-upper-bound} becomes
\[  \left\|   \widehat{\cD}^{(k)}   \widehat{\cD}^{(k)H}  \right\|_F^2  \leq  2    \left\|  \widehat{\cW}^{(k)}   \widehat{\cW}^{(k)H}
        - \widehat{\cW}_\star^{(k)}  \widehat{\cW}_\star^{(k)H} \right\|_F^2.
\]
Summing over \(k\) and using Parseval's identity gives
\begin{equation}
\label{eq:aligned-tubal-upper-bound}
    \|\cD*\cD^\ast\|_F^2
    \leq
    2
    \left\|
        \cW*\cW^\ast
        -
        \cW_\star*\cW_\star^\ast
    \right\|_F^2.
\end{equation} 
For the lower bound, \(\widehat{\cD}^{(k)}\) has at most \(r\) nonzero singular values. Thus
\begin{align*}
    \left\|
        \widehat{\cD}^{(k)}
        \widehat{\cD}^{(k)H}
    \right\|_F^2
    =
    \sum_{j=1}^{r}
    \sigma_j
    \bigl(
        \widehat{\cD}^{(k)}
    \bigr)^4
   \geq
    \frac1r
    \left(
        \sum_{j=1}^{r}
        \sigma_j
        \bigl(
            \widehat{\cD}^{(k)}
        \bigr)^2
    \right)^2 =
    \frac1r
    \left\|
        \widehat{\cD}^{(k)}
    \right\|_F^4.
\end{align*}
Using Parseval's identity and Jensen's inequality,
\begin{align}
    \|\cD*\cD^\ast\|_F^2=  \frac1{n_3}  \sum_{k=1}^{n_3}  \left\|    \widehat{\cD}^{(k)}    \widehat{\cD}^{(k)H}  \right\|_F^2\geq    \frac1{rn_3}  \sum_{k=1}^{n_3}  \left\|   \widehat{\cD}^{(k)}   \right\|_F^4\geq  \frac1r   \left(   \frac1{n_3}   \sum_{k=1}^{n_3}     \left\|    \widehat{\cD}^{(k)}    \right\|_F^2 \right)^2=   \frac1r \|\cD\|_F^4.
\label{eq:aligned-tubal-lower-bound}
\end{align}
Combining \eqref{eq:aligned-tubal-upper-bound} and \eqref{eq:aligned-tubal-lower-bound} proves \eqref{eq:aligned-factor-Gram}. Finally, \(\|\cD\|_F=\operatorname{dist}(\cW,\mathcal S_\star)\), so \eqref{eq:Gram-error-distance-lower0} follows immediately.
\end{proof}

\subsubsection{The aligned Hessian identity}


\begin{lemma}
\label{lem:aligned-Hessian}
Let
\( \mathcal A_{\cU} := \cU_\star*\mathcal Q_{\mathrm{opt}},    \mathcal A_{\cV}
    :=
    \cV_\star*\mathcal Q_{\mathrm{opt}},
\) 
and define the aligned factor errors
\[ \cD_{\cU}
    :=
    \cU-\mathcal A_{\cU},
    \qquad
    \cD_{\cV}
    :=
    \cV-\mathcal A_{\cV}.
\]
Then
\begin{align}
    & D^2f(\cU,\cV)
    \big[
        (\cD_{\cU},\cD_{\cV}),
        (\cD_{\cU},\cD_{\cV})
    \big]
    =\left\|
        \mathcal M
        \bigl(
            \cD_{\cU}* \cD_{\cV}^\ast
        \bigr)
    \right\|_2^2 +\frac14
    \left\| \cD_{\cU}^\ast* \cD_{\cU}  -\cD_{\cV}^\ast * \cD_{\cV}
    \right\|_F^2\notag\\
    &\qquad- 3 \left\|
        \mathcal M
        \bigl(
            \cU*\cV^\ast
            -
            \mathcal X_\star
        \bigr)
    \right\|_2^2
    -
    \frac34
    \left\|
        \cU^\ast*\cU
        -
        \cV^\ast*\cV
    \right\|_F^2
    +
    4Df(\cU,\cV)
    [
        \cD_{\cU},
        \cD_{\cV}
    ].
\label{eq:direct-aligned-Hessian}
\end{align}
In particular, at a critical point,
\begin{align}
    D^2f(\cU,\cV)
    \big[
        (\cD_{\cU},\cD_{\cV}),
        (\cD_{\cU},\cD_{\cV})
    \big]
    &=
    \left\|
        \mathcal M
        \bigl(
            \cD_{\cU}
            *
            \cD_{\cV}^\ast
        \bigr)
    \right\|_2^2
    +
    \frac14
    \left\|
        \cD_{\cU}^\ast
        *
        \cD_{\cU}
        -
        \cD_{\cV}^\ast
        *
        \cD_{\cV}
    \right\|_F^2
    \notag\\
    &\qquad
    -
    3
    \left\|
        \mathcal M
        \bigl(
            \cU*\cV^\ast
            -
            \mathcal X_\star
        \bigr)
    \right\|_2^2
    -
    \frac34
    \left\|
        \cU^\ast*\cU
        -
        \cV^\ast*\cV
    \right\|_F^2.
\label{eq:direct-aligned-Hessian-critical}
\end{align}
\end{lemma}

\begin{proof}
Consider the path
\(
    \cU(t)
    :=
    \cU+t\cD_{\cU}, 
    \cV(t)
    :=
    \cV+t\cD_{\cV}.
\)
Since
\( \mathcal A_{\cU}*\mathcal A_{\cV}^\ast =  \mathcal X_\star,
\)
we have 
\begin{align*}
    \cD_{\cU}  * \cV^\ast  +   \cU  *   \cD_{\cV}^\ast &= \cD_{\cU}  * (\mathcal A_{\mathcal{V}}+\cD_{\mathcal{V}})^\ast  +   (\mathcal A_{\cU}+\cD_{\cU})  *   \cD_{\cV}^\ast\\
    &= (\mathcal A_{\cU}+\cD_{\cU})  * (\mathcal A_{\mathcal{V}}+\cD_{\mathcal{V}})^\ast  -  \mathcal A_{\mathcal{U}} *   \mathcal A_{\cV}^\ast +\cD_{\cU}   *   \cD_{\cV}^\ast\\
   &=\mathcal{U}*\mathcal{V}^\ast-\mathcal{X}_\star+\cD_{\cU}   *   \cD_{\cV}^\ast=\mathcal E(\cU,\cV) +\cD_{\cU}   *   \cD_{\cV}^\ast. 
\end{align*} Therefore, we have that 
\begin{align*}
    \mathcal E  \bigl(  \cU(t),\cV(t)   \bigr)=& \mathcal E(\cU,\cV)  +  t  \left[  \cD_{\cU}  * \cV^\ast  +   \cU  *   \cD_{\cV}^\ast   \right]   +  t^2    \cD_{\cU}  *  \cD_{\cV}^\ast\notag\\
=&\mathcal E(\cU,\cV)  +  t  \left[  \mathcal E(\cU,\cV) +\cD_{\cU}   *   \cD_{\cV}^\ast \right]   +  t^2    \cD_{\cU}  *  \cD_{\cV}^\ast.
\end{align*}

Similarly, since the aligned ground-truth factors are balanced,
\(
    \mathcal A_{\cU}^\ast  *  \mathcal A_{\cU}    = \mathcal A_{\cV}^\ast *  \mathcal A_{\cV},
\)
we obtain
\begin{align*}
 \mathcal B  \bigl(   \cU(t),\cV(t)  \bigr)
    =
    \mathcal B(\cU,\cV)
    +  t \left[   \mathcal B(\cU,\cV)
        +  \cD_{\cU}^\ast  *  \cD_{\cU}
        - \cD_{\cV}^\ast   * \cD_{\cV}
    \right]  +  t^2  \left[    \cD_{\cU}^\ast
        *  \cD_{\cU}  -  \cD_{\cV}^\ast  *  \cD_{\cV}
    \right].
\end{align*}
Differentiating \(   f(\cU(t),\cV(t))  =  \frac12 \left\| \mathcal M   \bigl(  \mathcal E(\cU(t),\cV(t))  \bigr)  \right\|_2^2  +  \frac18  \left\|  \mathcal B(\cU(t),\cV(t))   \right\|_F^2 \)
at \(t=0\) gives
\begin{align}
 Df(\cU,\cV)
    [
        \cD_{\cU},
        \cD_{\cV}]
   =&
    \left\langle
        \mathcal M
        \bigl(
            \mathcal E(\cU,\cV)
        \bigr),
        \mathcal M
        \left(
            \mathcal E(\cU,\cV)
            +
            \cD_{\cU}
            *
            \cD_{\cV}^\ast
        \right)
    \right\rangle\notag\\
    &+
    \frac14
    \left\langle
        \mathcal B(\cU,\cV),
        \mathcal B(\cU,\cV)
        +
        \cD_{\cU}^\ast
        *
        \cD_{\cU}
        -
        \cD_{\cV}^\ast
        *
        \cD_{\cV}
    \right\rangle.
\label{eq:direct-aligned-first-derivative}
\end{align}
Differentiating a second time gives
\begin{align}
     D^2f(\cU,\cV)
    \big[
        (\cD_{\cU},\cD_{\cV}),
        (\cD_{\cU},&\cD_{\cV})
    \big]
   = \left\| \mathcal M  \left(  \mathcal E(\cU,\cV)
            + \cD_{\cU}  *  \cD_{\cV}^\ast
        \right) \right\|_2^2
    +  2   \left\langle  \mathcal M
        \bigl(     \mathcal E(\cU,\cV)   \bigr),
        \mathcal M  \bigl(   \cD_{\cU}      *   \cD_{\cV}^\ast  \bigr)
    \right\rangle  \notag\\
    & 
    +  \frac14   \left\|  \mathcal B(\cU,\cV)
        +  \cD_{\cU}^\ast  *  \cD_{\cU}
        - \cD_{\cV}^\ast
        *
        \cD_{\cV}
    \right\|_F^2    +
    \frac12
    \left\langle
        \mathcal B(\cU,\cV),
        \cD_{\cU}^\ast
        *
        \cD_{\cU}
        -
        \cD_{\cV}^\ast
        *
        \cD_{\cV}
    \right\rangle.
\label{eq:direct-aligned-second-derivative}
\end{align}
Expanding the squared norms in \eqref{eq:direct-aligned-second-derivative}, we obtain
\begin{align*}
     &D^2f(\cU,\cV)
    \big[ (\cD_{\cU},\cD_{\cV}),   (\cD_{\cU},\cD_{\cV})  \big]=
    \left\|  \mathcal M  \bigl(  \mathcal E(\cU,\cV)
        \bigr)  \right\|_2^2  +  4  \left\langle
        \mathcal M \bigl( \mathcal E(\cU,\cV)
        \bigr), \mathcal M  \bigl(  \cD_{\cU}
            *  \cD_{\cV}^\ast  \bigr)
    \right\rangle
    \\
  \qquad  &
    +   \left\|
        \mathcal M
        \bigl(
            \cD_{\cU}
            *
            \cD_{\cV}^\ast
        \bigr)
    \right\|_2^2
    +
    \frac14
    \left\|
        \mathcal B(\cU,\cV)
    \right\|_F^2
    +
    \left\langle
        \mathcal B(\cU,\cV),
        \cD_{\cU}^\ast
        *
        \cD_{\cU}
        -
        \cD_{\cV}^\ast
        *
        \cD_{\cV}
    \right\rangle
    +
    \frac14
    \left\|
        \cD_{\cU}^\ast
        *
        \cD_{\cU}
        -
        \cD_{\cV}^\ast
        *
        \cD_{\cV}
    \right\|_F^2.
\end{align*}
On the other hand, multiplying \eqref{eq:direct-aligned-first-derivative} by \(4\) gives
\begin{align*}
    4Df(\cU,\cV)
    [   \cD_{\cU}, \cD_{\cV}
    ]=&
    4
    \left\|
        \mathcal M
        \bigl(
            \mathcal E(\cU,\cV)
        \bigr)
    \right\|_2^2
    +
    4
    \left\langle
        \mathcal M
        \bigl(
            \mathcal E(\cU,\cV)
        \bigr),
        \mathcal M
        \bigl(
            \cD_{\cU}
            *
            \cD_{\cV}^\ast
        \bigr)
    \right\rangle
    \\
    &\qquad
    +
    \left\|
        \mathcal B(\cU,\cV)
    \right\|_F^2
    +
    \left\langle
        \mathcal B(\cU,\cV),
        \cD_{\cU}^\ast
        *
        \cD_{\cU}
        -
        \cD_{\cV}^\ast
        *
        \cD_{\cV}
    \right\rangle.
\end{align*}
Substituting this expression into the preceding expansion proves \eqref{eq:direct-aligned-Hessian}. At a critical point, the directional derivative vanishes, which gives \eqref{eq:direct-aligned-Hessian-critical}.
\end{proof}
\subsubsection{Proof of the global landscape theorem}

\begin{proof}[Proof of Theorem~\ref{thm:global-strict-saddle}]

{Every element of \(\mathcal S_\star\) has zero reconstruction residual and zero balancing residual, and hence has objective value zero. Since \(f\geq0\), every element of \(\mathcal S_\star\) is a global minimizer. The reverse inclusion will follow from the negative-curvature argument below.}

We next derive a curvature estimate that is valid at every point. We begin with the two positive terms in the aligned Hessian identity. Parseval's identity and the corresponding matrix trace identities on each Fourier slice give
\begin{align}
    &\left\|
        \cD_{\cU}^\ast*\cD_{\cU}
        -
        \cD_{\cV}^\ast*\cD_{\cV}
    \right\|_F^2
    =
    \left\|
        \cD_{\cU}*\cD_{\cU}^\ast
    \right\|_F^2
    +
    \left\|
        \cD_{\cV}*\cD_{\cV}^\ast
    \right\|_F^2
    -
    2
    \left\|
        \cD_{\cU}*\cD_{\cV}^\ast
    \right\|_F^2.
\label{eq:direct-balance-error-identity}
\end{align}
On the other hand, the block structure of \(\cD*\cD^\ast\) gives
\begin{align}
    \|\cD*\cD^\ast\|_F^2
    &=
    \left\|
        \cD_{\cU}*\cD_{\cU}^\ast
    \right\|_F^2
    +
    \left\|
        \cD_{\cV}*\cD_{\cV}^\ast
    \right\|_F^2
    +  2     \left\|
        \cD_{\cU}*\cD_{\cV}^\ast
    \right\|_F^2.
\label{eq:stacked-quadratic-error}
\end{align}
Therefore, we have 
\begin{equation*}
 \left\|      \cD*\cD^\ast \right\|_F^2=4    \left\| \cD_{\cU}*\cD_{\cV}^\ast \right\|_F^2+\left\|
        \cD_{\cU}^\ast*\cD_{\cU}-  \cD_{\cV}^\ast*\cD_{\cV}
    \right\|_F^2.
\end{equation*}
Hence \(4    \left\| \cD_{\cU}*\cD_{\cV}^\ast \right\|_F^2\leq  \left\|      \cD*\cD^\ast \right\|_F^2 \). 
Since \(  \operatorname{rank}_{\mathrm t}  \bigl(  \cD_{\cU}*   \cD_{\cV}^\ast  \bigr)  \leq r, \)
the t-RIP, together with \eqref{eq:direct-balance-error-identity} and \eqref{eq:stacked-quadratic-error}, yields
\begin{align}
     4
    \left\|  \mathcal M   \bigl(    \cD_{\cU}*
            \cD_{\cV}^\ast  \bigr) \right\|_2^2
    +  \left\|   \cD_{\cU}^\ast*\cD_{\cU}  -  \cD_{\cV}^\ast*\cD_{\cV}  \right\|_F^2
    \leq&
    4(1+\delta)
    \left\|
        \cD_{\cU}*
        \cD_{\cV}^\ast
    \right\|_F^2
    +
    \left\|
        \cD_{\cU}^\ast*\cD_{\cU}
        -
        \cD_{\cV}^\ast*\cD_{\cV}
    \right\|_F^2
    \notag\\
    =&
    \|\cD*\cD^\ast\|_F^2
    +
    4\delta
    \left\|
        \cD_{\cU}*
        \cD_{\cV}^\ast
    \right\|_F^2
    \leq
    (1+\delta)
    \|\cD*\cD^\ast\|_F^2.
\label{eq:positive-terms-t-RIP}
\end{align}

For notational convenience, define the aligned ground-truth factor by
\[
\mathcal A
:=
\cW_\star*\mathcal Q_{\mathrm{opt}}
=
\begin{bmatrix}
\mathcal A_{\cU}\\
\mathcal A_{\cV}
\end{bmatrix}.
\]
We now estimate the two negative terms in the aligned Hessian identity. By the definitions of \(\cW\) and \(\mathcal A\),
\begin{equation*}
    \cW*\cW^\ast
    -
    \mathcal A*\mathcal A^\ast
    =
    \begin{bmatrix}
        \cU*\cU^\ast
        -
        \mathcal A_{\cU}*
        \mathcal A_{\cU}^\ast
        &
        \cU*\cV^\ast-\mathcal X_\star
        \\
        \cV*\cU^\ast-\mathcal X_\star^\ast
        &
        \cV*\cV^\ast
        -
        \mathcal A_{\cV}*
        \mathcal A_{\cV}^\ast
    \end{bmatrix}.
\end{equation*}
Using
\( \mathcal A_{\cU}*   \mathcal A_{\cV}^\ast
    =  \mathcal X_\star
\)
and
\(  \mathcal A_{\cU}^\ast*  \mathcal A_{\cU}  = \mathcal A_{\cV}^\ast*   \mathcal A_{\cV}, \)
a direct expansion gives
\begin{align*}   &4 \left\|  \cU*\cV^\ast-\mathcal X_\star \right\|_F^2  + \left\| \cU^\ast*\cU    -    \cV^\ast*\cV \right\|_F^2 =   \left\|   \cW*\cW^\ast   -  \mathcal A*\mathcal A^\ast \right\|_F^2    +     2    \left\|        \mathcal A_{\cU}^\ast*\cU      -    \mathcal A_{\cV}^\ast*\cV  \right\|_F^2.\end{align*}
In particular,
\begin{equation*}
    4  \left\|   \cU*\cV^\ast-\mathcal X_\star   \right\|_F^2 + \left\|  \cU^\ast*\cU  -  \cV^\ast*\cV  \right\|_F^2  \geq  \left\|     \cW*\cW^\ast  - \mathcal A*\mathcal A^\ast  \right\|_F^2.
\end{equation*}
The reconstruction residual has tubal rank at most \(2r\), so t-RIP gives
\begin{align}
    4 \left\|   \mathcal M \bigl(  \cU*\cV^\ast-\mathcal X_\star  \bigr) \right\|_2^2+\left\| \cU^\ast*\cU- \cV^\ast*\cV\right\|_F^2 \geq&4(1-\delta) \left\| \cU*\cV^\ast-\mathcal X_\star \right\|_F^2 + \left\|    \cU^\ast*\cU-   \cV^\ast*\cV \right\|_F^2 \notag\\
    \geq&(1-\delta) (4 \left\|   \cU*\cV^\ast-\mathcal X_\star   \right\|_{F}^2+\left\| \cU^\ast*\cU- \cV^\ast*\cV\right\|_F^2 )\notag\\ 
    \geq&(1-\delta)  \left\|   \cW*\cW^\ast-    \mathcal A*\mathcal A^\ast \right\|_F^2. \label{eq:negative-terms-t-RIP}
\end{align}

Applying Lemma~\ref{lem:aligned-Hessian} and then using \eqref{eq:positive-terms-t-RIP} and \eqref{eq:negative-terms-t-RIP}, we obtain
\begin{align}
    D^2f(\cU,\cV)
    \big[   (\cD_{\cU},\cD_{\cV}),   (\cD_{\cU},\cD_{\cV})
    \big]  \leq&
    \frac{1+\delta}{4}
    \|\cD*\cD^\ast\|_F^2
    -
    \frac{3(1-\delta)}{4}
    \left\|
        \cW*\cW^\ast
        -
        \mathcal A*\mathcal A^\ast
    \right\|_F^2\notag\\
   &  + 4
    Df(\cU,\cV)
    [
        \cD_{\cU},
        \cD_{\cV}
    ].
\label{eq:direct-Hessian-after-t-RIP}
\end{align}
Lemma~\ref{lem:aligned-factor-Gram} gives that
\(  \|\cD*\cD^\ast\|_F^2 \leq  2
    \left\|   \cW*\cW^\ast    -   \mathcal A*\mathcal A^\ast   \right\|_F^2.\)
Substituting this into \eqref{eq:direct-Hessian-after-t-RIP} yields
\begin{align}
    D^2f(\cU,\cV)
    \big[
        (\cD_{\cU},\cD_{\cV}),
        (\cD_{\cU},\cD_{\cV})
    \big]
    &\leq
    -\frac{1-5\delta}{4}
    \left\|
        \cW*\cW^\ast
        -
        \mathcal A*\mathcal A^\ast
    \right\|_F^2
    +
    4
    Df(\cU,\cV)
    [
        \cD_{\cU},
        \cD_{\cV}
    ].
    \label{eq:direct-master-curvature-delta}
\end{align}

 By Lemma~\ref{lem:aligned-factor-Gram}, we have  
\begin{equation}
\label{eq:Gram-error-distance-lower}
    \left\|
        \cW*\cW^\ast
        -
        \mathcal A*\mathcal A^\ast
    \right\|_F^2
    \geq
    \frac1{2r}
    \|\cD\|_F^4.
\end{equation}

Suppose now that \((\cU,\cV)\notin\mathcal S_\star\) is a critical point. Then \(\cD\neq0\) and
\(  Df(\cU,\cV)  [   \cD_{\cU},   \cD_{\cV} ]  = 0. \)
It follows from \eqref{eq:direct-master-curvature-delta} and \eqref{eq:Gram-error-distance-lower} that
\begin{align*}
    &D^2f(\cU,\cV)  \big[     (\cD_{\cU},\cD_{\cV}),    (\cD_{\cU},\cD_{\cV})  \big]  \leq  -\frac{1-5\delta}{8r}   \|\cD\|_F^4.
\end{align*}
Dividing by \(\|\cD\|_F^2\) and using \(  \|\cD\|_F   =  \operatorname{dist}\bigl(   \cW,\mathcal S_\star   \bigr)\), we conclude that
\begin{equation*}
    \lambda_{\min}
    \bigl(
        \nabla^2f(\cU,\cV)
    \bigr)
    \leq
    -\frac{(1-5\delta)
        \operatorname{dist}
        \bigl(
            (\cU,\cV),\mathcal S_\star
        \bigr)^2
    }{8r}.
\end{equation*}
Since $\delta<\frac{1}{5}$, we have \(  \lambda_{\min}
    \bigl(
        \nabla^2f(\cU,\cV)
    \bigr)<0\). 
Thus every critical point outside \(\mathcal S_\star\) is a strict saddle. Every local minimizer is a critical point with positive-semidefinite Hessian and must therefore belong to \(\mathcal S_\star\). Conversely, every point of \(\mathcal S_\star\) has objective value zero and is globally minimizing. Thus every local minimum is global, and the set of global
minimizers is exactly \(\mathcal S_\star\). This proves parts~(i) and~(ii).

It remains to prove the quantitative strict-saddle property.
Fix \(\epsilon,\gamma,\zeta>0\) satisfying  
\begin{equation}
\label{eq:strict-saddle-parameter-condition}
    \gamma+\frac{4\epsilon}{\zeta}
    \leq
    \frac{1-5\delta}{8r}\,\zeta^2.
\end{equation}
Consider any factor pair \(\cW=(\cU,\cV)\) satisfying
\[
    \|\nabla f(\cU,\cV)\|_F<\epsilon,
    \qquad
    \operatorname{dist}(\cW,\mathcal S_\star)>\zeta.
\]
Let \(\cD=\cW-\mathcal A\) be the aligned factor error, where
\(\mathcal A=\cW_\star*\mathcal Q_{\mathrm{opt}}\), and set
\[
    d:=\|\cD\|_F
    =
    \operatorname{dist}(\cW,\mathcal S_\star)
    >
    \zeta.
\]
By Cauchy--Schwarz,
\[
    Df(\cU,\cV)[\cD_{\cU},\cD_{\cV}]
    \leq
    \|\nabla f(\cU,\cV)\|_F\|\cD\|_F
    <
    \epsilon d.
\]
Since  \(\delta<1/5\),  combining
\eqref{eq:direct-master-curvature-delta} with
\eqref{eq:Gram-error-distance-lower} gives
\[
    D^2f(\cU,\cV)[\cD,\cD]
    \leq
    -\frac{1-5\delta}{8r}\,d^4
    +
    4Df(\cU,\cV)[\cD_{\cU},\cD_{\cV}]
    <
    -\frac{1-5\delta}{8r}\,d^4
    +
    4\epsilon d.
\]
Because \(d>0\), the Rayleigh quotient therefore yields
\begin{equation*}
    \lambda_{\min}\bigl(\nabla^2f(\cU,\cV)\bigr)
    \leq
    \frac{D^2f(\cU,\cV)[\cD,\cD]}{d^2}
    <
    -\frac{1-5\delta}{8r}\,d^2
    +
    \frac{4\epsilon}{d}.
\end{equation*}
Using \(d>\zeta\), \(1-5\delta>0\), and
\eqref{eq:strict-saddle-parameter-condition}, we obtain
\[
    \lambda_{\min}\bigl(\nabla^2f(\cU,\cV)\bigr)
    <
    -\frac{1-5\delta}{8r}\,\zeta^2
    +
    \frac{4\epsilon}{\zeta}
    \leq
    -\gamma.
\]
Thus, whenever both the large-gradient and proximity alternatives
fail, the negative-curvature alternative holds. Consequently,
every factor pair \(\cW=(\cU,\cV)\) satisfies at least one of
\[
    \|\nabla f(\cU,\cV)\|_F\geq\epsilon,
    \qquad
    \lambda_{\min}\bigl(\nabla^2f(\cU,\cV)\bigr)\leq-\gamma,
    \qquad
    \operatorname{dist}(\cW,\mathcal S_\star)\leq\zeta.
\]
Hence \(f\) is \((\epsilon,\gamma,\zeta)\)-strict saddle
relative to \(\mathcal S_\star\). The explicit parameters are obtained by setting \(\gamma=4\epsilon/\zeta\) and imposing equality in the parameter condition.

\end{proof}
\subsection{Multi-Rank-Dependent Landscape Geometry}
\label{subsec:rank-profile-dichotomy}

Theorem~\ref{thm:global-strict-saddle} establishes a quantitative strict-saddle estimate for arbitrary multi-ranks. We now sharpen this result by showing that a uniform multi-rank yields the matrix-like scale, whereas a nonuniform multi-rank produces additional quartically flat directions.

Define the smallest positive singular value among the Fourier slices of \(\mathcal X_\star\) by
\begin{equation*}
    \sigma_\star
    :=
    \min_{\substack{k\in[n_3]\\ \rho_k>0}}
    \sigma_{\rho_k}
    \bigl(
        \widehat{\mathcal X}_\star^{(k)}
    \bigr).
\end{equation*}

\begin{theorem}
\label{thm:rank-profile-dichotomy}
Assume the hypotheses of Theorem~\ref{thm:global-strict-saddle}.

\begin{enumerate}[label=(\roman*),leftmargin=2em]

    \item 


Suppose that \(\rho_k=r\) for every \(k\in[n_3]\). 
For any \(\epsilon,\gamma,\zeta>0\) satisfying
\begin{equation}
\label{eq:uniform-profile-parameter-condition}
    \gamma+\frac{4\epsilon}{\zeta}  \leq  (1-5\delta)(\sqrt2-1)\sigma_\star,
\end{equation}
the objective \(f\) is
\begin{equation}   \label{eq:uniform-profile-strict-saddle}
(\epsilon,\gamma,\zeta)\text{-strict saddle relative to~}\mathcal S_\star.
\end{equation}
In particular, for every \(\epsilon>0\), one may take
\[
\gamma=\frac{(1-5\delta)(\sqrt2-1)\sigma_\star}{2},
\qquad
\zeta=\frac{8\epsilon}{(1-5\delta)(\sqrt2-1)\sigma_\star}.
\]

    Moreover, at every \(\overline{\cW}=(\overline{\cU},\overline{\cV})\in\mathcal S_\star\),
    \begin{equation}
    \label{eq:uniform-Hessian-kernel}
        \ker\nabla^2f(\overline{\cW})
        =
        T_{\overline{\cW}}\mathcal S_\star
        =
        \left\{
            \bigl(
                \overline{\cU}*\mathcal K,\,
                \overline{\cV}*\mathcal K
            \bigr):
            \mathcal K^\ast=-\mathcal K
        \right\}.
    \end{equation}
    In particular, the objective has quadratic growth transverse to the solution set.

    \item Suppose that \(\rho_{k_0}<r\) for some \(k_0\in[n_3]\). Then, at every \(\overline{\cW}\in\mathcal S_\star\), there exists a unit direction
    \[
        \cD
        =
        \bigl(
            \cD_{\cU},
            \cD_{\cV}
        \bigr)
        \perp
        T_{\overline{\cW}}\mathcal S_\star
    \]
    and constants \(a_{\cD},b_{\cD}>0\) such that
    \begin{equation}
    \label{eq:heterogeneous-flat-curve}
        \operatorname{dist}
        \bigl(
            \overline{\cW}+t\cD,
            \mathcal S_\star
        \bigr)
        =
        |t|,
        \qquad
        f(\overline{\cW}+t\cD)
        =
        a_{\cD}t^4,
        \qquad
        \left\|
            \nabla f(\overline{\cW}+t\cD)
        \right\|_F
        =
        b_{\cD}|t|^3
    \end{equation}
    for all sufficiently small \(t\).

    Consequently,
    \(  T_{\overline{\cW}}\mathcal S_\star  \subsetneq \ker\nabla^2f(\overline{\cW}),  \)
    and neither local quadratic growth nor a local Polyak--\L{}ojasiewicz inequality holds around \(\mathcal S_\star\). Moreover, for any fixed \(\gamma_0>0\), the objective cannot be \((\epsilon,\gamma_0,\zeta(\epsilon))\)-strict saddle for all sufficiently small \(\epsilon\) if  \( \zeta(\epsilon) = o(\epsilon^{1/3}).  \)
\end{enumerate}
\end{theorem}

The proof of the uniform-profile conclusion requires the following strengthening of Lemma~\ref{lem:aligned-factor-Gram}.

\begin{lemma}
\label{lem:uniform-profile-Procrustes}
Suppose that \(\rho_k=r\) for every \(k\in[n_3]\). Then
\begin{equation}
\label{eq:uniform-Procrustes-lower-bound}
    \left\|
        \cW*\cW^\ast
        -
        \cW_\star*\cW_\star^\ast
    \right\|_F^2
    \geq
    4(\sqrt2-1)\sigma_\star
    \operatorname{dist}
    \bigl(
        \cW,\mathcal S_\star
    \bigr)^2.
\end{equation}
\end{lemma}

\begin{proof}
Let \(\mathcal Q_{\mathrm{opt}}\) and \(\cD\) be the Procrustes minimizer and aligned error introduced in Subsection~\ref{subsubsec:tubal-Procrustes}. For each \(k\in[n_3]\), set
\(
    W_k   := \widehat{\cW}^{(k)},
    A_k:= \widehat{\cW}_\star^{(k)} \widehat{\mathcal Q}_{\mathrm{opt}}^{(k)}, 
    D_k := W_k-A_k.
\)
The Procrustes condition gives \(A_k^HW_k\succeq0\). Since \(\rho_k=r\), the matrix \(A_k\) has full column rank. 
 By the standard realification argument, \cite[Lemma~5.4]{tu2016low} extends to complex matrices and gives 
\begin{equation}
\label{eq:uniform-slice-Procrustes-lower}
    \|W_kW_k^H-A_kA_k^H\|_F^2
    \geq
    2(\sqrt2-1)
    \sigma_{\min}(A_k)^2
    \|D_k\|_F^2.
\end{equation}
Because the ground-truth factors are balanced, \(A_k^HA_k\) is unitarily similar to \(2\Sigma_k\), where \(\Sigma_k\) contains the singular values of \(\widehat{\mathcal X}_\star^{(k)}\). Therefore \( \sigma_{\min}(A_k)^2 =2\sigma_r \bigl(\widehat{\mathcal X}_\star^{(k)}\bigr) \geq 2\sigma_\star.\)
Substituting this into \eqref{eq:uniform-slice-Procrustes-lower}, summing over \(k\), and applying Parseval's identity proves \eqref{eq:uniform-Procrustes-lower-bound}.
\end{proof}

\begin{proof}[Proof of Theorem~\ref{thm:rank-profile-dichotomy}]
\begin{itemize}
    \item [(i)]
We first assume that the Fourier rank profile is uniform. Let
\(  d  := \operatorname{dist} \bigl(\cW,\mathcal S_\star \bigr)= \|\cD\|_F.\)
Suppose that \(\|\nabla f(\cW)\|_F<\epsilon\). By Cauchy--Schwarz,
\(  Df(\cU,\cV)  [ \cD_{\cU}, \cD_{\cV} ] \leq \epsilon d.
\) 
Combining the Gram-level curvature estimate \eqref{eq:direct-master-curvature-delta} with Lemma~\ref{lem:uniform-profile-Procrustes} gives
\begin{equation}
\label{eq:uniform-profile-curvature}
    D^2f(\cW) [   \cD,\cD ]
    \leq -(\sqrt2-1)(1-5\delta)
    \sigma_\star d^2
    +
    4\epsilon d.
\end{equation}
If \(d>\zeta\), then dividing \eqref{eq:uniform-profile-curvature} by \(d^2\) yields \(\lambda_{\min}\bigl(   \nabla^2f(\cW) \bigr) \leq  -(\sqrt2-1)(1-5\delta)
    \sigma_\star +4\epsilon/\zeta.\) According to \eqref{eq:uniform-profile-parameter-condition}, we thus have \(\lambda_{\min}\bigl(   \nabla^2f(\cW) \bigr) \leq-\gamma\).  This proves \eqref{eq:uniform-profile-strict-saddle}.  The explicit parameters are obtained by setting \(\gamma=4\epsilon/\zeta\) and imposing equality in the parameter condition.

The same Gram estimate also gives quadratic growth. Indeed, the residual estimate established in the proof of Theorem~\ref{thm:global-strict-saddle} implies
\(f(\cW)\geq \frac{1-\delta}{8} \left\|   \cW*\cW^\ast - \cW_\star*\cW_\star^\ast \right\|_F^2.\) 
Hence Lemma~\ref{lem:uniform-profile-Procrustes} gives
\begin{equation*}
    f(\cW)
    \geq
    \frac{(1-\delta)(\sqrt2-1)}{2}
    \sigma_\star
    \operatorname{dist}
    \bigl(
        \cW,\mathcal S_\star
    \bigr)^2.
\end{equation*}
It remains to identify the Hessian kernel. Let \(\overline{\cW}=(\overline{\cU},\overline{\cV})\in\mathcal S_\star\), and let \(\mathcal Z=(\mathcal Z_{\cU},\mathcal Z_{\cV})\). Since $\mathcal{E}(\overline{\cW})=0$ and $\mathcal{B}(\overline{\cW})=0$,
\begin{align}
    D^2f(\overline{\cW})
    [\mathcal Z,\mathcal Z]
    &=
    \left\|
        \mathcal M
        \bigl(
            \mathcal Z_{\cU}*\overline{\cV}^{\ast}
            +
            \overline{\cU}*\mathcal Z_{\cV}^{\ast}
        \bigr)
    \right\|_2^2
    +
    \frac14
    \left\|
        \overline{\cU}^{\ast}*\mathcal Z_{\cU}
        +
        \mathcal Z_{\cU}^{\ast}*\overline{\cU}
        -
        \overline{\cV}^{\ast}*\mathcal Z_{\cV}
        -
        \mathcal Z_{\cV}^{\ast}*\overline{\cV}
    \right\|_F^2.
\label{eq:Hessian-at-global-minimizer}
\end{align}
Every tangent direction in \eqref{eq:uniform-Hessian-kernel} makes both terms vanish, so \(  T_{\overline{\cW}}\mathcal S_\star\subseteq \ker\nabla^2f(\overline{\cW}).\)
\\
Conversely, suppose that \(\mathcal Z\) belongs to the Hessian kernel. Since the tensor in the first term of \eqref{eq:Hessian-at-global-minimizer} has tubal rank at most \(2r\), t-RIP implies
\( \mathcal Z_{\cU}*\overline{\cV}^{\ast} + \overline{\cU}*\mathcal Z_{\cV}^{\ast} = 0.\)
At each Fourier frequency, \(\widehat{\overline{\cU}}^{(k)}\) and \(\widehat{\overline{\cV}}^{(k)}\) have full column rank. Projecting the last identity onto the orthogonal complements of their column spaces shows that
\[
    \widehat{\mathcal Z}_{\cU}^{(k)}
    =
    \widehat{\overline{\cU}}^{(k)}K_k,
    \qquad
    \widehat{\mathcal Z}_{\cV}^{(k)}
    =
    \widehat{\overline{\cV}}^{(k)}L_k
\]
for some \(K_k,L_k\in\mathbb C^{r\times r}\), and the product identity gives   \( K_k+L_k^H=0.\) 
Let
\[ G_k:= \widehat{\overline{\cU}}^{(k)H} \widehat{\overline{\cU}}^{(k)} = \widehat{\overline{\cV}}^{(k)H} \widehat{\overline{\cV}}^{(k)} \succ0.\]
The vanishing of the second term in \eqref{eq:Hessian-at-global-minimizer} then gives
\[ G_k(K_k+K_k^H)  + (K_k+K_k^H)G_k  = 0.\] 
Set \(S_k:=K_k+K_k^H\). Since \(S_k=S_k^H\), taking the real Frobenius inner product with \(S_k\) yields
\[
\begin{aligned}
0   &=    \langle  S_k, G_kS_k+S_kG_k \rangle_F = 2\operatorname{tr}(S_kG_kS_k)= 2\|G_k^{1/2}S_k\|_F^2.
\end{aligned}
\]
Because \(G_k\succ0\), its square root is invertible, so \(S_k=0\). Thus \(K_k^H=-K_k\). Hence \(L_k=K_k\). The matrices \(K_k\) satisfy the required conjugate-symmetry conditions and therefore assemble into a real tensor \(\mathcal K\) satisfying \(\mathcal K^\ast=-\mathcal K\). This proves \eqref{eq:uniform-Hessian-kernel} and completes part~(i).
\item [(ii)]
We now assume that \(\rho_{k_0}<r\) for some \(k_0\). By the t-orthogonal invariance of the objective, it suffices to construct the flat direction at the canonical balanced factorization \((\cU_\star,\cV_\star)\). Let \(\widehat{\cX}_\star^{(k_0)}=P_{k_0}\Sigma_{k_0}Q_{k_0}^H\) be a compact SVD, where \(P_{k_0}\) and \(Q_{k_0}\) each have \(\rho_{k_0}\) orthonormal columns.
Choose the SVD gauge in which
\[
    \widehat{\cU}_\star^{(k_0)}
    =
    P_{k_0}\Sigma_{k_0}^{1/2}
    \begin{bmatrix}
        I_{\rho_{k_0}} & 0
    \end{bmatrix},
    \qquad
    \widehat{\cV}_\star^{(k_0)}
    =
    Q_{k_0}\Sigma_{k_0}^{1/2}
    \begin{bmatrix}
        I_{\rho_{k_0}} & 0
    \end{bmatrix}.
\]
Choose \(j\in\{\rho_{k_0}+1,\ldots,r\}\) and a unit vector \(a_{k_0}\perp\operatorname{range}(P_{k_0})\), and define
\[
    \widehat{\cD}_{\cU}^{(k_0)}
    :=
    a_{k_0}e_j^T,
    \qquad
    \widehat{\cD}_{\cV}^{(k_0)}
    :=
    0.
\]
At the conjugate frequency, use the complex-conjugate slice, and set all remaining slices equal to zero. At a self-conjugate frequency, choose \(a_{k_0}\) real. After normalization, this defines a real unit direction \(\cD=(\cD_{\cU},0)\) satisfying
\begin{equation}
\label{eq:inactive-direction-identities}
    \cD_{\cU}*\cV_\star^\ast
    =
    0,
    \qquad
    \cU_\star^\ast*\cD_{\cU}
    =
    0.
\end{equation}
Consider \( \cW(t) := \bigl(   \cU_\star+t\cD_{\cU},   \cV_\star\bigr).\)
The first identity in \eqref{eq:inactive-direction-identities} shows that the reconstruction residual vanishes identically along this curve. Since the ground-truth factors are balanced and the perturbation lies in an inactive factor direction,
\( \mathcal B\bigl(\cW(t)\bigr)= t^2 \cD_{\cU}^{\ast} * \cD_{\cU}.
\) 
Consequently,
\begin{equation*}
 f\bigl(\cW(t)\bigr) =  \frac{t^4}{8} \left\| \cD_{\cU}^{\ast}  * \cD_{\cU}  \right\|_F^2, \quad  \nabla_{\cU}f\bigl(\cW(t)\bigr) =  \frac{t^3}{2}  \cD_{\cU}  * \bigl(    \cD_{\cU}^{\ast}   *    \cD_{\cU} \bigr), \quad    \nabla_{\cV}f\bigl(\cW(t)\bigr)   = 0. \end{equation*}
This gives the objective and gradient identities in \eqref{eq:heterogeneous-flat-curve}.

At frequency \(k_0\), every tensor of the form \(\widehat{\cU}_\star^{(k_0)}\widehat{\mathcal Q}^{(k_0)}\) has its columns in \(\operatorname{range}(P_{k_0})\), whereas \(\widehat{\cD}_{\cU}^{(k_0)}\) is orthogonal to this space. It follows that 
\[\operatorname{dist}\bigl(      \cW(t),\mathcal S_\star\bigr) = |t|.\] 
The same orthogonality shows that \(\cD\perp T_{\cW_\star}\mathcal S_\star\). 
Substituting $\cD$ and $\mathcal{W}_\star$ into \eqref{eq:Hessian-at-global-minimizer} gives 
  \(D^2f(\cW_\star) [\cD,\cD]  = 0. \) 
The Hessian is positive semidefinite at the global minimizer, so \(\cD\in\ker\nabla^2f(\cW_\star)\). Transporting the direction by the common t-orthogonal action gives the corresponding conclusion at every point of \(\mathcal S_\star\).

Finally,
\[  \frac{   f(\cW(t))  }{    \operatorname{dist}      (\cW(t),\mathcal S_\star)^2   }  =   a_{\cD}t^2  \longrightarrow0, \qquad  \frac{   \|\nabla f(\cW(t))\|_F^2  }{  f(\cW(t))  } =  \frac{b_{\cD}^2}{a_{\cD}}t^2  \longrightarrow0. \]
Thus local quadratic growth and the local Polyak--\L{}ojasiewicz inequality both fail.

We finally show that the cubic-root proximity scale is unavoidable in the heterogeneous-profile case. Fix \(\gamma_0>0\). Since \(\overline{\cW}\) is a global minimizer, \(\nabla^2f(\overline{\cW})\) is positive semidefinite. Moreover, the normal direction \(\cD\) constructed above belongs to its kernel. Hence
\(  \lambda_{\min}  \bigl( \nabla^2f(\overline{\cW})    \bigr)   =     0.
\)
The Hessian depends continuously on the factor pair. Therefore, there exists \(t_0>0\) such that
\( \lambda_{\min}  \bigl(  \nabla^2f(\cW(t)) \bigr) > -\gamma_0\) whenever \(|t|<t_0\).

Recall from \eqref{eq:heterogeneous-flat-curve} that, for all sufficiently small \(t\),
\[
    \left\|
        \nabla f(\cW(t))
    \right\|_F
    =
    b_{\cD}|t|^3,
    \qquad
    \operatorname{dist}
    \bigl(
        \cW(t),\mathcal S_\star
    \bigr)
    =
    |t|,
\]
where \(b_{\cD}>0\) is independent of \(t\). Given \(\epsilon>0\), set \( t_\epsilon   :=  \left(    \frac{\epsilon}{2b_{\cD}}
    \right)^{1/3}.
\)
This value is chosen so that the gradient norm along the flat curve is strictly smaller than the threshold \(\epsilon\). Indeed,
\begin{align}
    \left\|  \nabla f  \bigl(\cW(t_\epsilon)  \bigr) \right\|_F= b_{\cD}t_\epsilon^3= b_{\cD}\left(    \frac{\epsilon}{2b_{\cD}} \right) = \frac{\epsilon}{2}<\epsilon.
\label{eq:gradient-at-t-epsilon}
\end{align}
At the same time,
\begin{align}
    \operatorname{dist}
    \bigl( \cW(t_\epsilon), \mathcal S_\star \bigr)=
    t_\epsilon=\left(
        \frac{\epsilon}{2b_{\cD}}
    \right)^{1/3}= (2b_{\cD})^{-1/3}  \epsilon^{1/3}.
\label{eq:distance-at-t-epsilon}
\end{align}

Since \(t_\epsilon\to0\) as \(\epsilon\to0\), for all sufficiently small \(\epsilon\) we have \(t_\epsilon<t_0\). Therefore,
\begin{equation}
\label{eq:Hessian-at-t-epsilon}
    \lambda_{\min}
    \bigl(
        \nabla^2f
        \bigl(
            \cW(t_\epsilon)
        \bigr)
    \bigr)
    >
    -\gamma_0.
\end{equation}

Suppose that \(f\) were \((\epsilon,\gamma_0,\zeta(\epsilon))\)-strict saddle relative to \(\mathcal S_\star\). At the point \(\cW(t_\epsilon)\),   \eqref{eq:gradient-at-t-epsilon} shows that the large-gradient alternative fails, while \eqref{eq:Hessian-at-t-epsilon} shows that the negative-curvature alternative fails. Therefore, the proximity alternative must hold:
\[
    \operatorname{dist}
    \bigl(
        \cW(t_\epsilon),
        \mathcal S_\star
    \bigr)
    \leq
    \zeta(\epsilon).
\]
Using \eqref{eq:distance-at-t-epsilon}, we conclude that
\begin{equation*}
    \zeta(\epsilon)
    \geq
    (2b_{\cD})^{-1/3}
    \epsilon^{1/3}
\end{equation*}
for all sufficiently small \(\epsilon\). In particular,
\[
    \liminf_{\epsilon\to0}
    \frac{\zeta(\epsilon)}{\epsilon^{1/3}}
    \geq
    (2b_{\cD})^{-1/3}
    >
    0.
\]
Hence no proximity radius satisfying
\(\zeta(\epsilon) =  o(\epsilon^{1/3}) \)
can yield an \((\epsilon,\gamma_0,\zeta(\epsilon))\)-strict-saddle property for all sufficiently small \(\epsilon\). 
\end{itemize}
\end{proof}

\section{Numerical Illustration}
\label{sec:numerics}

The results of Section~\ref{sec:main-results} are geometric: they describe the
shape of the objective \eqref{eq:balanced-sensing-objective} over the factor
space rather than the behavior of any particular algorithm. We first visualize the benign global geometry of
Theorem~\ref{thm:global-strict-saddle} and the multi-rank–dependent local geometry
of Theorem~\ref{thm:rank-profile-dichotomy}. We then complement these geometric
illustrations with a real-OCT tensor-sensing experiment.

\subsection{Synthetic experimental setup}
\label{subsec:synthetic-setup}
We first use synthetic examples to illustrate the global and local geometric properties established in Section~\ref{sec:main-results}. These experiments are designed to visualize the objective over carefully chosen low-dimensional slices of the factor space. Throughout this subsection,
we set \(  n_1=n_2=4,~   n_3=3,\)
and consider tubal ranks \(r\in\{1,2\}\).

\vspace{0.1in}
\noindent{\bf Construction of the ground truth.}
We construct a balanced factorization
\(\mathcal X_\star = \mathcal U_\star\ast\mathcal V_\star^{*},
    ~
    \mathcal U_\star^{*}\ast\mathcal U_\star= \mathcal V_\star^{*}\ast\mathcal V_\star,\)
in the Fourier domain. At each frequency \(k\), we generate orthonormal matrices \(P_k\) and \(Q_k\) and a nonnegative diagonal matrix \(\Sigma_k\), and set
\[\widehat{\mathcal U}_\star^{(k)} = P_k\Sigma_k^{1/2},
    \qquad
    \widehat{\mathcal V}_\star^{(k)} = Q_k\Sigma_k^{1/2}.
\]
The factors are chosen to satisfy the conjugate-symmetry conditions required for their inverse Fourier transforms to be real. In particular, the factors at paired frequencies are complex conjugates, while those at self-conjugate frequencies are real. 
For the uniform-profile examples, every Fourier slice has rank \(r\). For the nonuniform-profile example, the final diagonal entry of \(\Sigma_{k_0}\) is set to zero at the self-conjugate frequency \(k_0=1\).  Consequently, \(  \rho_{k_0}=r-1,     \max_k\rho_k=r.\)
Thus the tensor still has tubal rank \(r\), although its width-\(r\)
factorization is overparameterized at frequency \(k_0\).

\vspace{0.1in}
\noindent{\bf Objective.}
To isolate the intrinsic factorization geometry from finite-sample effects, we take the sensing operator to be the identity. The objective  is therefore
\[
    f(\mathcal U,\mathcal V) =\frac12 \left\|   \mathcal U\ast\mathcal V^{*}-\mathcal X_\star \right\|_F^2 + \frac18\left\|    \mathcal U^{*}\ast\mathcal U   -    \mathcal V^{*}\ast\mathcal V\right\|_F^2.
\]
The identity operator satisfies the \((2r,0)\)-t-RIP, so this objective is a special case of the setting analyzed in Theorems~\ref{thm:global-strict-saddle} and~\ref{thm:rank-profile-dichotomy}.

\vspace{0.1in}
\noindent{\bf Tangent and normal directions.}
Let
\(
    \overline{\mathcal W} :=
    \begin{bmatrix}
        \mathcal U_\star\\
        \mathcal V_\star
    \end{bmatrix} \in\mathcal S_\star.
\) 
The tangent space of the solution orbit at
\(\overline{\mathcal W}\) is
\(
    T_{\overline{\mathcal W}}\mathcal S_\star =  \left\{
        \begin{bmatrix}
            \mathcal U_\star\ast\mathcal K\\
            \mathcal V_\star\ast\mathcal K
        \end{bmatrix} : \mathcal K^{*}=-\mathcal K\right\}.
\)
We construct this space explicitly in the Fourier domain. At a
self-conjugate frequency, the corresponding Fourier slice of
\(\mathcal K\) is real skew-symmetric. At a paired frequency, it is
skew-Hermitian, and the slice at the conjugate frequency is chosen as its
complex conjugate. The resulting real tangent directions are orthonormalized
by an SVD.

Given an initial direction \(\mathcal D_0\), we form a unit normal direction
by
\[
    \mathcal D
    =
    \frac{
        \mathcal D_0
        -
        P_{T_{\overline{\mathcal W}}\mathcal S_\star}\mathcal D_0
    }{
        \left\|
        \mathcal D_0
        -
        P_{T_{\overline{\mathcal W}}\mathcal S_\star}\mathcal D_0
        \right\|_F
    }.
\]
Hence every direction described below as normal is numerically orthogonal to
the full tangent space of the t-orthogonal solution orbit. 

\subsection{Visualization of the benign global landscape}
\label{subsec:synthetic-global}

We first consider a uniform rank-one ground truth. Define
\(
    s:=\|\overline{\mathcal W}\|_F
    ~\text{and}~
    \mathcal R:=\frac{\overline{\mathcal W}}{s}.
\)
The direction \(\mathcal R\) is radial and passes through both the selected
ground-truth factor and the rank-deficient origin. We generate a random
factor-space direction, remove its component along \(\mathcal R\), project
the result onto
\(N_{\overline{\mathcal W}}\mathcal S_\star\), and normalize it. Denote the
resulting transverse normal direction by \(\mathcal D_\perp\). We evaluate the objective on the two-dimensional slice
\[
    \mathcal W_{\mathrm{glob}}(\xi,\eta)
    =
    \xi\overline{\mathcal W}
    +
    \eta s\mathcal D_\perp,
    \qquad
    (\xi,\eta)\in[-1.6,1.6]\times [-1.6,1.6].
\]
Here \(\xi\) is the radial coordinate and \(\eta\) is a transverse normal
coordinate. The objective is evaluated on a \(151\times151\) grid.

Figure~\ref{fig:benign-landscape} shows the resulting surface and contour
plots. The points
\(
    (\xi,\eta)=(1,0)
    ~\text{and}~
    (\xi,\eta)=(-1,0)
\)
correspond to
\(\overline{\mathcal W}\) and \(-\overline{\mathcal W}\), respectively,
and both belong to the global-minimizer orbit \(\mathcal S_\star\). These
are two representatives of the orbit.  
The origin is also a critical point, since both factor tensors vanish there,
but it is not a global minimizer. The displayed slice places the origin on a
ridge between the two marked global-minimizer representatives. This is
consistent with Theorem~\ref{thm:global-strict-saddle}, which identifies every nonglobal critical point
in the full factor space as a strict saddle. The figure is intended only as
a low-dimensional illustration of this geometry and does not enumerate all
critical points of the full objective.

\begin{figure}[ht]
    \centering
\includegraphics[width=0.4\textwidth]{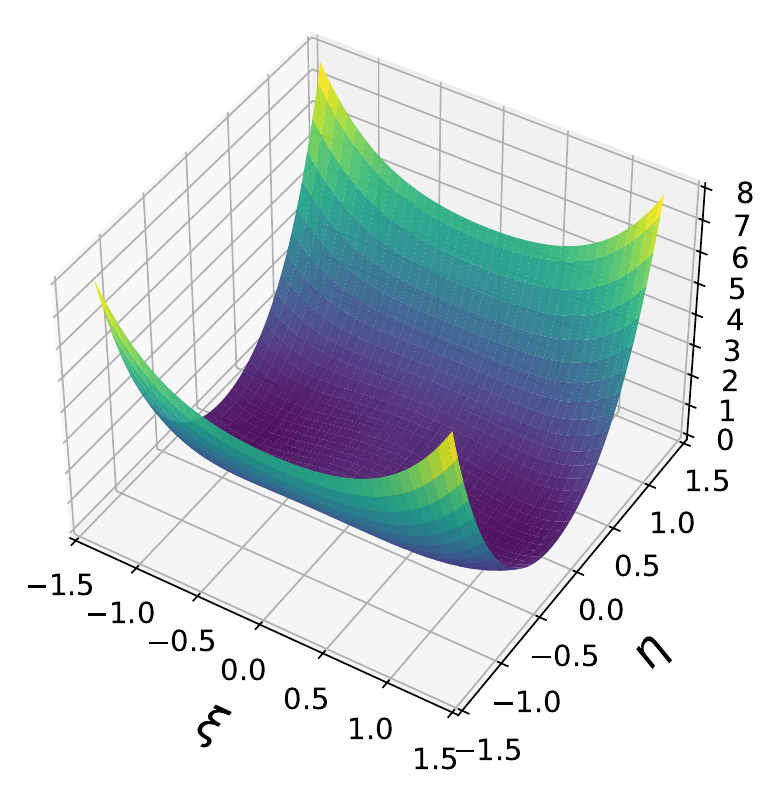}
\includegraphics[width=0.45\textwidth]{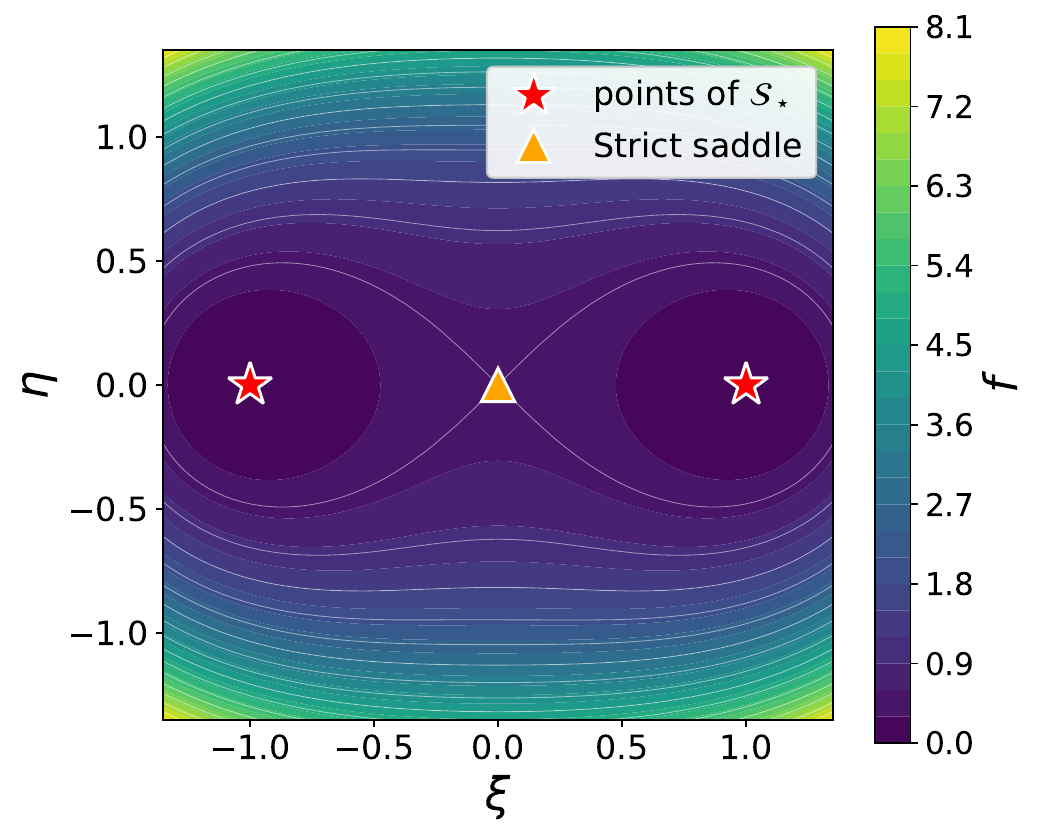}
    \caption{
    Benign global geometry of the balanced objective for a uniform
    rank-one ground truth with
    \(n_1=n_2=4\), \(n_3=3\), and \(r=1\).
    The objective is evaluated on the slice
    \(\mathcal W_{\mathrm{glob}}(\xi,\eta)= \xi\overline{\mathcal W}+ \eta\|\overline{\mathcal W}\|_F\mathcal D_\perp,\)
    where \(\xi\) is the radial coordinate and \(\eta\) is a transverse  normal coordinate.
    \textnormal{Left:} Surface plot. \textnormal{Right:} Contour plot.
    The stars at \((1,0)\) and \((-1,0)\) mark
    \(\overline{\mathcal W}\) and \(-\overline{\mathcal W}\), two points
    of the positive-dimensional solution orbit \(\mathcal S_\star\).
    The triangle at the origin marks the rank-deficient strict saddle.
    }
    \label{fig:benign-landscape}
\end{figure}

\subsection{Multi-rank-dependent local geometry}
\label{subsec:synthetic-multirank}

We next illustrate the geometric dichotomy in Theorem~\ref{thm:rank-profile-dichotomy}. We compare two
ground truths with the same tubal rank \(r=2\). The first has the uniform multi-rank profile, i.e.,  \(\rho_k=2\) for every \(k\), 
whereas the second contains one deficient Fourier slice, \(\rho_{1}=1.\)

For each ground truth, let \(\mathcal T\) be a unit vector in \(T_{\overline{\mathcal W}}\mathcal S_\star\), and let \(\mathcal D\) be a unit vector in \(N_{\overline{\mathcal W}}\mathcal S_\star\). We evaluate the objective on the affine tangent--normal slice
\[
    \mathcal W_{\mathrm{loc}}(a,b)
    =
    \overline{\mathcal W}
    +
    \|\overline{\mathcal W}\|_F
    \bigl(a\mathcal T+b\mathcal D\bigr),
    \qquad
    (a,b)\in[-0.5,0.5]\times [-0.5,0.5] .
\]
The contour plots use a \(121\times121\) grid. The line corresponding to \(b=0\) is the affine tangent line at \(\overline{\mathcal W}\); it is not the exact curved orbit
\(\overline{\mathcal W}\ast\mathcal Q(a)\).

For the uniform-profile example, we begin with a random perturbation supported at the self-conjugate frequency \(k_0=1\) and project it onto the normal space. Since every Fourier slice has full column rank, Theorem~\ref{thm:rank-profile-dichotomy}(i) shows that the Hessian kernel at
\(\overline{\mathcal W}\) consists exactly of the tangent directions. Consequently, the Hessian is strictly positive along every nonzero normal direction, and
\[
    f(\overline{\mathcal W}+t\mathcal D)=c_2t^2+O(t^3),  c_2>0.
\]
Thus the objective has quadratic growth transverse to the solution orbit.

For the deficient-profile example, we use the inactive-coordinate direction appearing in the proof of Theorem~\ref{thm:rank-profile-dichotomy}(ii). In the Fourier domain, the direction is supported only at \(k_0=1\) and has the form
\(
    \widehat{\mathcal D}_{U}^{(k_0)}= a e_r^{\mathsf T}, 
    \widehat{\mathcal D}_{V}^{(k_0)}=0,
\)
where \( a\perp \operatorname{range}  \bigl(\widehat{\mathcal U}_\star^{(k_0)}\bigr).
\)
Because \(k_0\) is self-conjugate, \(a\) is chosen to be real. All remaining Fourier slices of the direction vanish. This construction gives
\(
    \mathcal D_U\ast\mathcal V_\star^{*}=0,
    \mathcal U_\star^{*}\ast\mathcal D_U=0.
\)
The direction is normal to the solution orbit in exact arithmetic. 

Along the curve
\(  \mathcal W(t)  =
    \bigl(
        \mathcal U_\star+t\mathcal D_U,
        \mathcal V_\star
    \bigr),\)
the reconstruction residual vanishes identically:
\(
    (\mathcal U_\star+t\mathcal D_U)
    \ast\mathcal V_\star^{*}
    -
    \mathcal X_\star
    =0.
\)
The balancing residual is
\(  \mathcal B(\mathcal W(t))  =   t^2\mathcal D_U^{*}\ast\mathcal D_U,
\)
and therefore
\(  f(\mathcal W(t))  =
    \frac{t^4}{8}
    \left\|
        \mathcal D_U^{*}\ast\mathcal D_U
    \right\|_F^2.
\)
Hence the quartic growth along this direction is exact. Moreover,
\( \|\nabla f(\mathcal W(t))\|_F =  \Theta(|t|^3),
\)
which agrees with the cubic gradient growth established in
Theorem~\ref{thm:rank-profile-dichotomy}(ii).

Figure~\ref{fig:multirank-geometry}(a)--(b) compares the corresponding
tangent--normal slices. The uniform-profile objective rises sharply in the
normal direction, whereas the deficient-profile objective exhibits a
visibly flatter valley. The tangent coordinate is included to display the
orbit-induced degeneracy common to both cases; the distinction between the
two profiles occurs in the normal direction.

To quantify the transverse growth, we evaluate
\(  f(\overline{\mathcal W}+t\mathcal D)
\)
at \(30\) logarithmically spaced values satisfying
\(
    10^{-2.5}\leq t\leq10^{-1.2}.
\)
A least-squares line is then fitted to
\(\log f(\overline{\mathcal W}+t\mathcal D)\) as a function of
\(\log t\). As shown in
Figure~\ref{fig:multirank-geometry}(c), the fitted exponent is close to
\(2\) for the uniform multi-rank and close to \(4\) for the deficient
multi-rank. The restricted fitting interval emphasizes the local regime
and reduces the influence of higher-order terms on the uniform-profile
fit.

\begin{figure}[ht]
    \centering
  \includegraphics[width=0.32\textwidth]{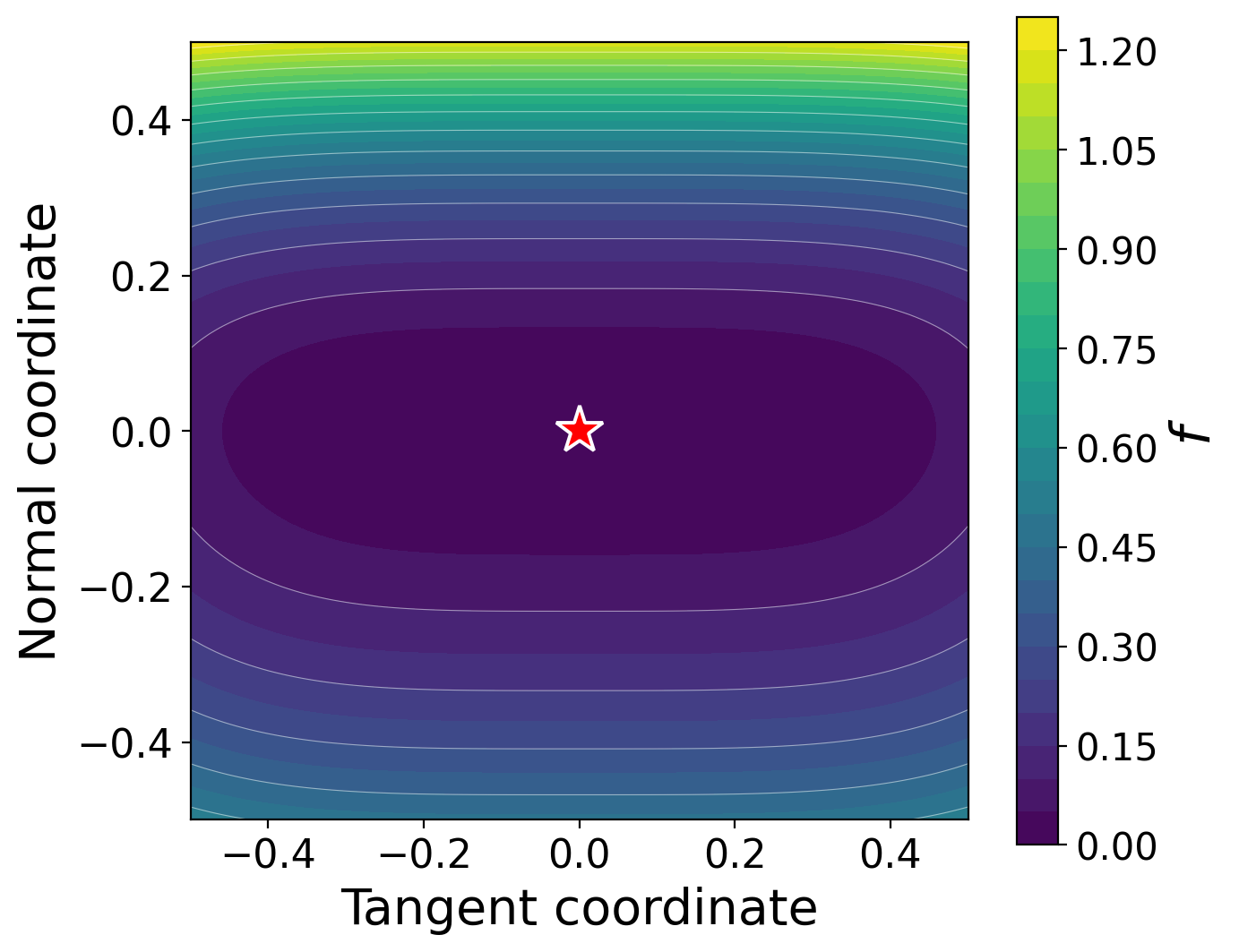}
    \includegraphics[width=0.32\textwidth]{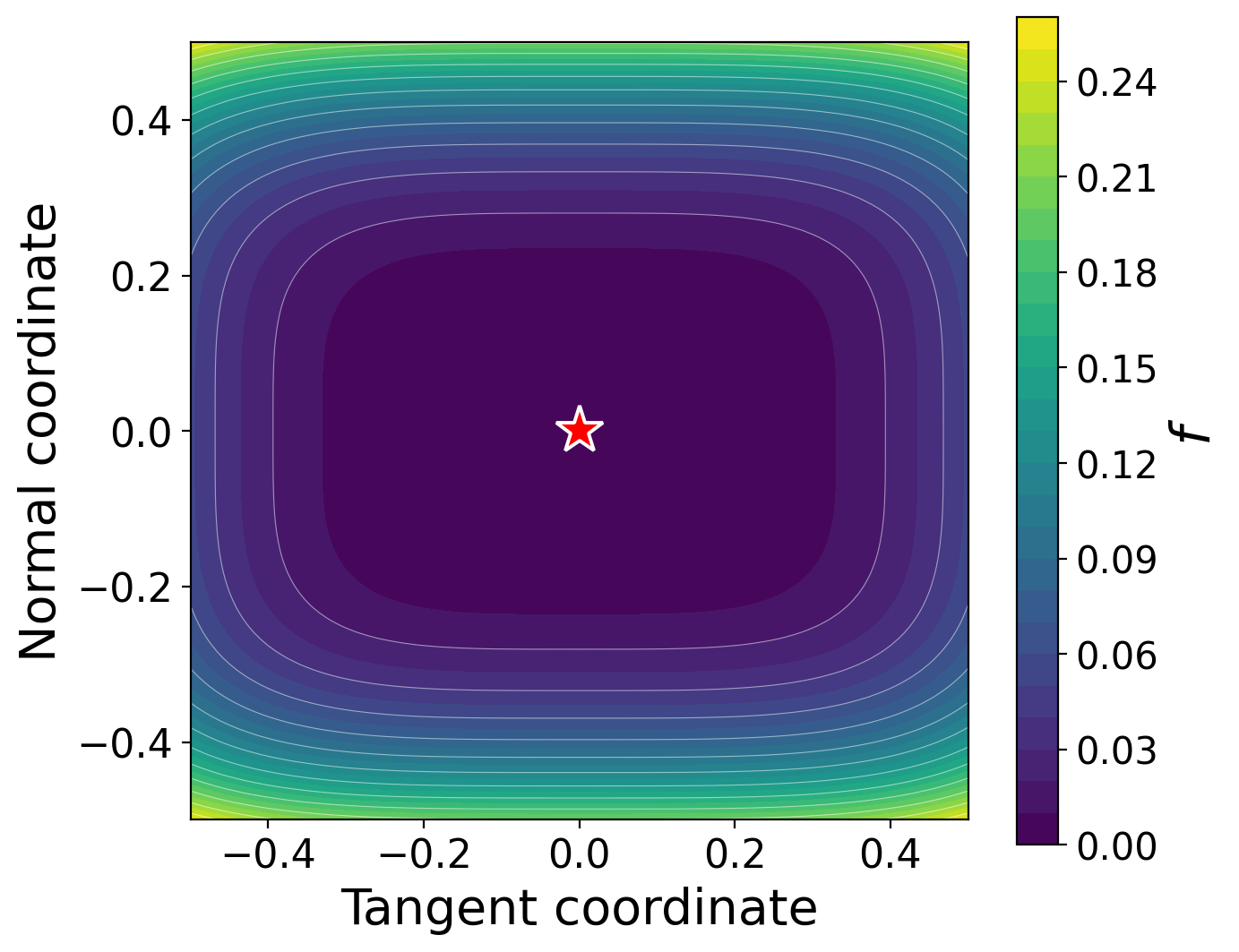}
    \includegraphics[width=0.31\textwidth]{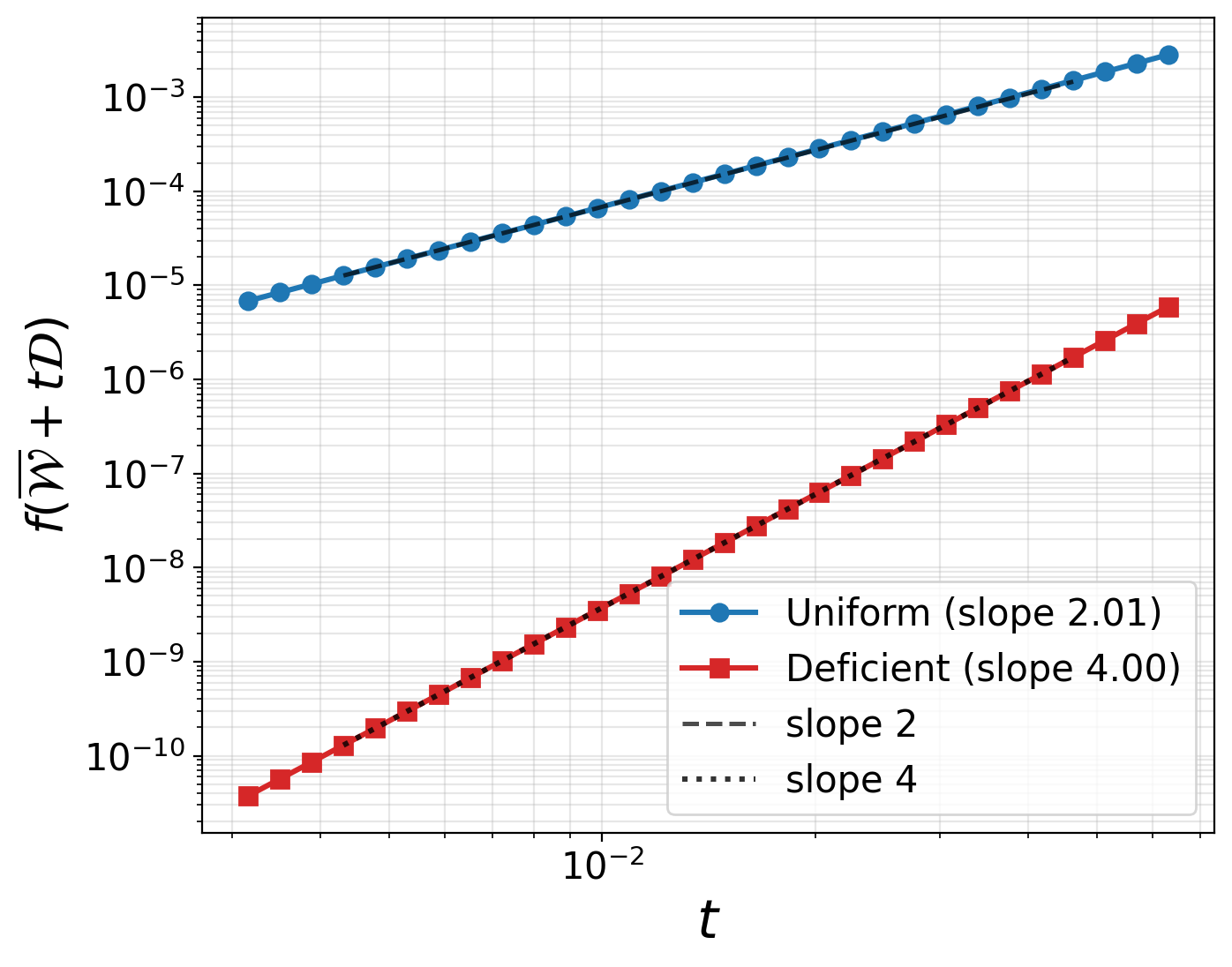}
    \caption{
    Multi-rank-dependent local geometry at a selected global minimizer
    \(\overline{\mathcal W}\), for two ground truths with equal tubal rank
    \(r=2\) and dimensions \(n_1=n_2=4\), \(n_3=3\).
    \textnormal{(a)} Uniform multi-rank,
    \(\rho_k=r\) for every \(k\).
    \textnormal{(b)} Nonuniform multi-rank with one deficient slice,
    \(\rho_{k_0}<r\).
    The horizontal and vertical axes in panels \textbf{Left} and \textbf{Middle}  are the tangent and normal  coordinates, respectively, and the star marks
    \(\overline{\mathcal W}\).
    \textbf{Right:} Log--log plot of
    \(f(\overline{\mathcal W}+t\mathcal D)\) along the selected normal
    direction. The uniform and deficient cases follow slopes close to
    \(2\) and \(4\), respectively, illustrating the quadratic and quartic
    growth laws in Theorem~\ref{thm:rank-profile-dichotomy}.
    }
    \label{fig:multirank-geometry}
\end{figure}

These synthetic experiments illustrate two complementary aspects of the
theory. The global slice in Figure~\ref{fig:benign-landscape} displays the
benign strict-saddle structure established in Theorem~\ref{thm:global-strict-saddle}, while Figure~\ref{fig:multirank-geometry} isolates the local distinction identified in Theorem~\ref{thm:rank-profile-dichotomy}. In particular, tensors with the same tubal rank can exhibit different transverse growth orders because the fine geometry is governed by the complete multi-rank profile rather than by the tubal rank alone.

\subsection[Real-OCT tensor-sensing experiment]{Real-OCT tensor-sensing experiment}
\label{sec:oct-experiment}

Motivated by the benign-landscape guarantee in
Theorem~\ref{thm:global-strict-saddle}, we examine the recovery behavior
of the balanced factorized formulation from random initializations using a
real retinal optical coherence tomography (OCT) volume and simulated linear
measurements. We compare an exact-low-tubal-rank target with its untruncated
counterpart to investigate how recovery changes when the exact-rank modeling
assumption is relaxed.  We use
the publicly downloadable UK Biobank OCT example (Resource~337), which
contains 128 grayscale B-scans of size $650\times512$.\footnote{Available at
\url{https://biobank.ndph.ox.ac.uk/ukb/refer.cgi?id=337}.}  We estimate the
axial retinal support from the volume-averaged row-intensity profile, retain
rows 171--532, and remove 3\% from each lateral boundary.  The resulting stack
is reordered as lateral $\times$ slow-scan $\times$ axial and trilinearly
resampled to
\(
    n_1\times n_2\times n_3=28\times8\times28,
    ~ N=n_1n_2n_3=6272.
\)
Intensities are linearly rescaled using the empirical 0.5th and 99.5th
percentiles and clipped to $[0,1]$.  We denote the resulting untruncated
preprocessed tensor by $\mathcal{X}_{\mathrm{raw}}$.

We consider two targets.  For the exact-rank experiment, we let
$\mathcal{X}_\star$ be the rank-$r$ truncated t-SVD of
$\mathcal{X}_{\mathrm{raw}}$, with $r=3$.  All 28 Fourier slices of
$\mathcal{X}_\star$ have numerical rank three under the relative tolerance
$10^{-10}$, so $\operatorname{rank}_{\mathrm t}(\mathcal{X}_\star)=3$ and
the factor width equals the true tubal rank.  For the model-mismatch
experiment, the target is $\mathcal{X}_{\mathrm{raw}}$, while the factor width
remains 3.  Its best rank-3 t-SVD approximation error is
\(
    \frac{\|\mathcal{X}_{\mathrm{raw}}-\mathcal{X}_\star\|_F}
         {\|\mathcal{X}_{\mathrm{raw}}\|_F}=0.0513.
\)
The first target satisfies the exact-rank modeling assumption,
whereas the second introduces low-rank model mismatch at the
same factor width.
We denote the target by $\mathcal X_{\mathrm{tar}}$, which equals
$\mathcal X_\star$ in the exact-rank experiment and
$\mathcal X_{\mathrm{raw}}$ in the model-mismatch experiment.

For each measurement count $m$, we generate a normalized Gaussian sensing
operator
\(
    [\mathcal{M}(\mathcal{Z})]_\ell
      =\langle\mathcal{A}_\ell,\mathcal{Z}\rangle_F,
    \text{~with~}
    [\mathcal{A}_\ell]_{ijk}\overset{\mathrm{i.i.d.}}{\sim}
      \mathcal{N}(0,1/m),
\)
and form noiseless measurements $y=\mathcal{M}(\mathcal{X}_{\mathrm{tar}})$.
We use $m/N\in\{0.2,0.3,\ldots,0.8\}$ with 
$m\in\{1254,1882,2509,3136,3763,4390,5018\}$.\footnote{For these dimensions, no linear sensing map with
$m<N$ can satisfy $(6,\delta)$-t-RIP with $\delta<1/2$.
Both experiments therefore examine empirical recovery outside
the small-t-RIP regime covered by
Theorem~\ref{thm:global-strict-saddle}.}  At each measurement ratio, we use the same Gaussian sensing operator
for both targets.   
We minimize the balanced objective
\eqref{eq:balanced-sensing-objective} with $\mathcal X_\star$ replaced by $\mathcal X_{\mathrm{tar}}$. At each $m/N$, we use ten independent random initializations constructed by frequency-wise QR factorizations.  Conjugate symmetry is imposed explicitly,
and the two factors have identical Gram tensors at initialization; the largest
normalized initial balance residual over all trials is $2.1\times10^{-16}$.
The random initializations use scales of $0.5$, $1$, and $2$
times a reference factor norm derived from $\mathcal M^\ast(y)$.
We optimize with MATLAB's quasi-Newton implementation of
\texttt{fminunc}, using the analytic gradient, at most 150 iterations, and
optimality tolerance $10^{-6}$.\footnote{The theorem concerns the geometry of the
objective and does not require this particular solver or initialization.}
\begin{figure}[t]
    \centering
    \includegraphics[width=0.8\textwidth]{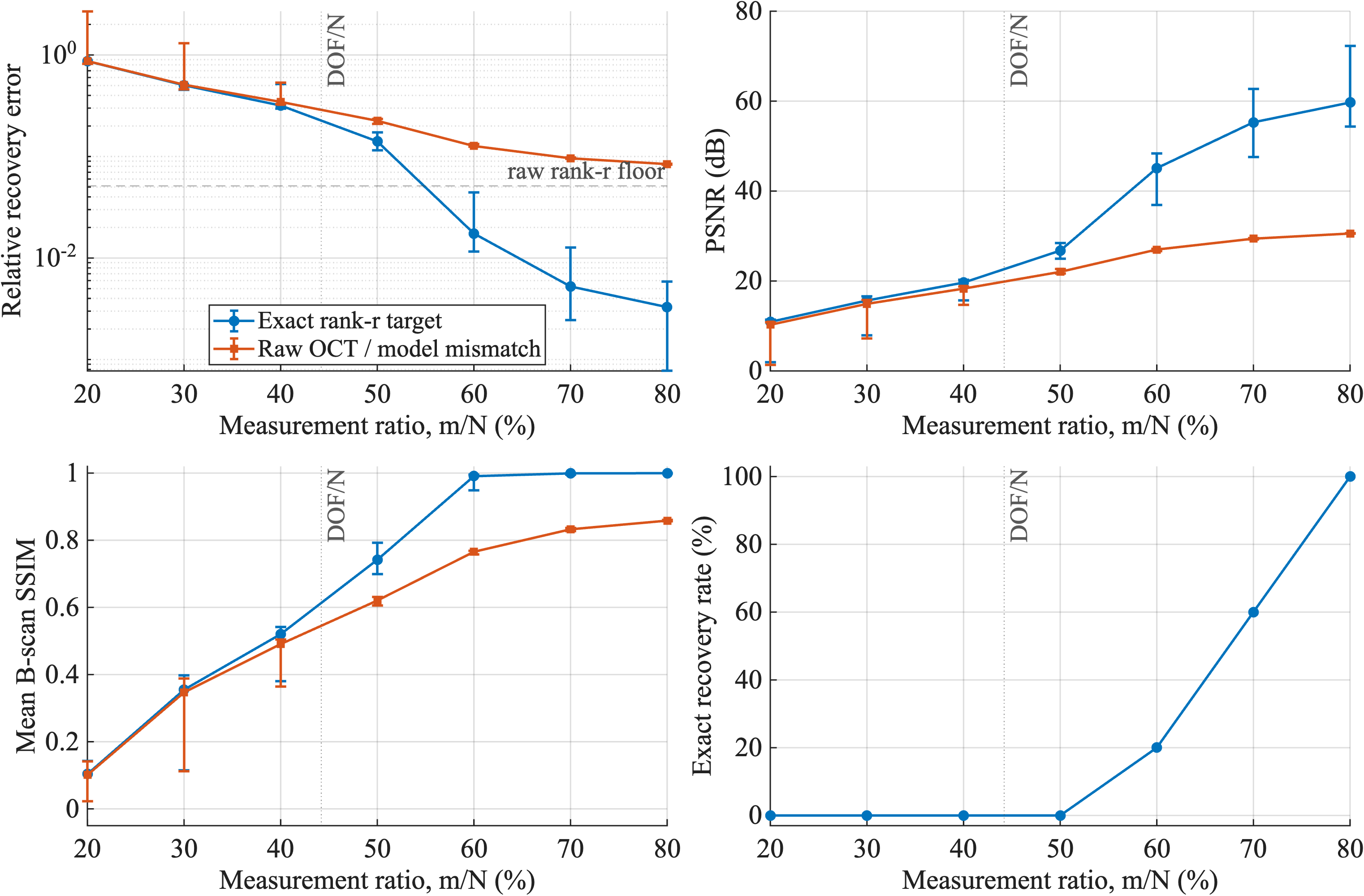}
    \caption{The first three panels show
    medians over ten balanced random starts, with error bars spanning the
    25th to 75th percentiles. The dashed horizontal line is the rank-3
    approximation floor of the untruncated target. The vertical line (DOF/N) at
    $r(n_1+n_2-r)n_3/N=0.4420$ is only a parameter-count reference, not a
    theoretical phase-transition threshold. The lower-right panel shows the percentage of successful starts for the exact-rank target.}
    \label{fig:oct-summary}
\end{figure}

\begin{figure}[ht]
    \centering
    \includegraphics[width=\textwidth]{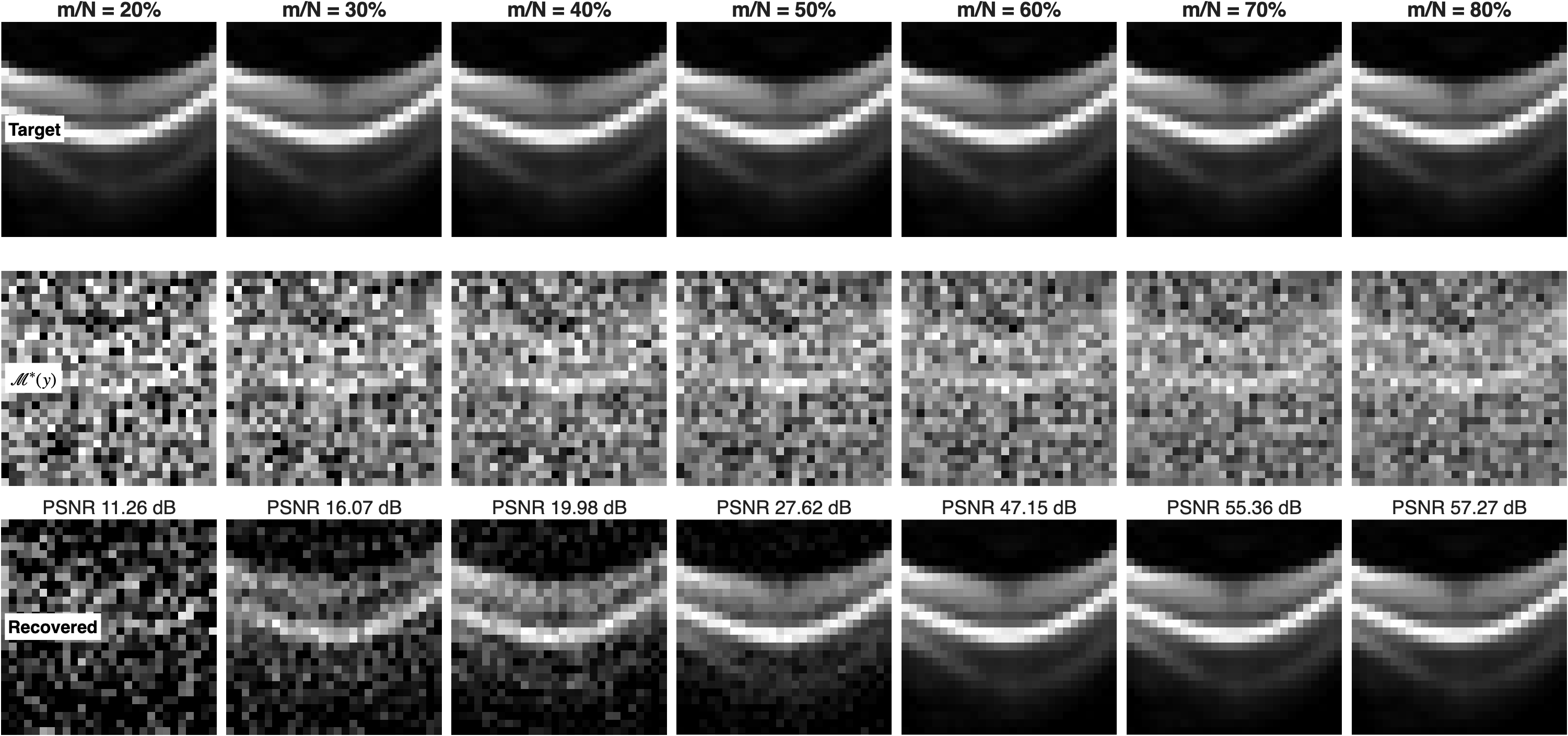}
    \caption{Exact-rank OCT-derived target under normalized Gaussian sensing.
    Columns correspond to $m/N=20\%,\ldots,80\%$. Rows show the target, the
    adjoint backprojection $\mathcal{M}^{*}(y)$, and a representative
    reconstruction.  The number above each reconstruction is that
    individual trial's full-volume PSNR (in dB). For each measurement ratio,
    we select the trial whose full-volume relative recovery error is closest
    to the median over ten balanced random starts.}
    \label{fig:oct-exact}
\end{figure}

\begin{figure}[!h]
    \centering
    \includegraphics[width=\textwidth]{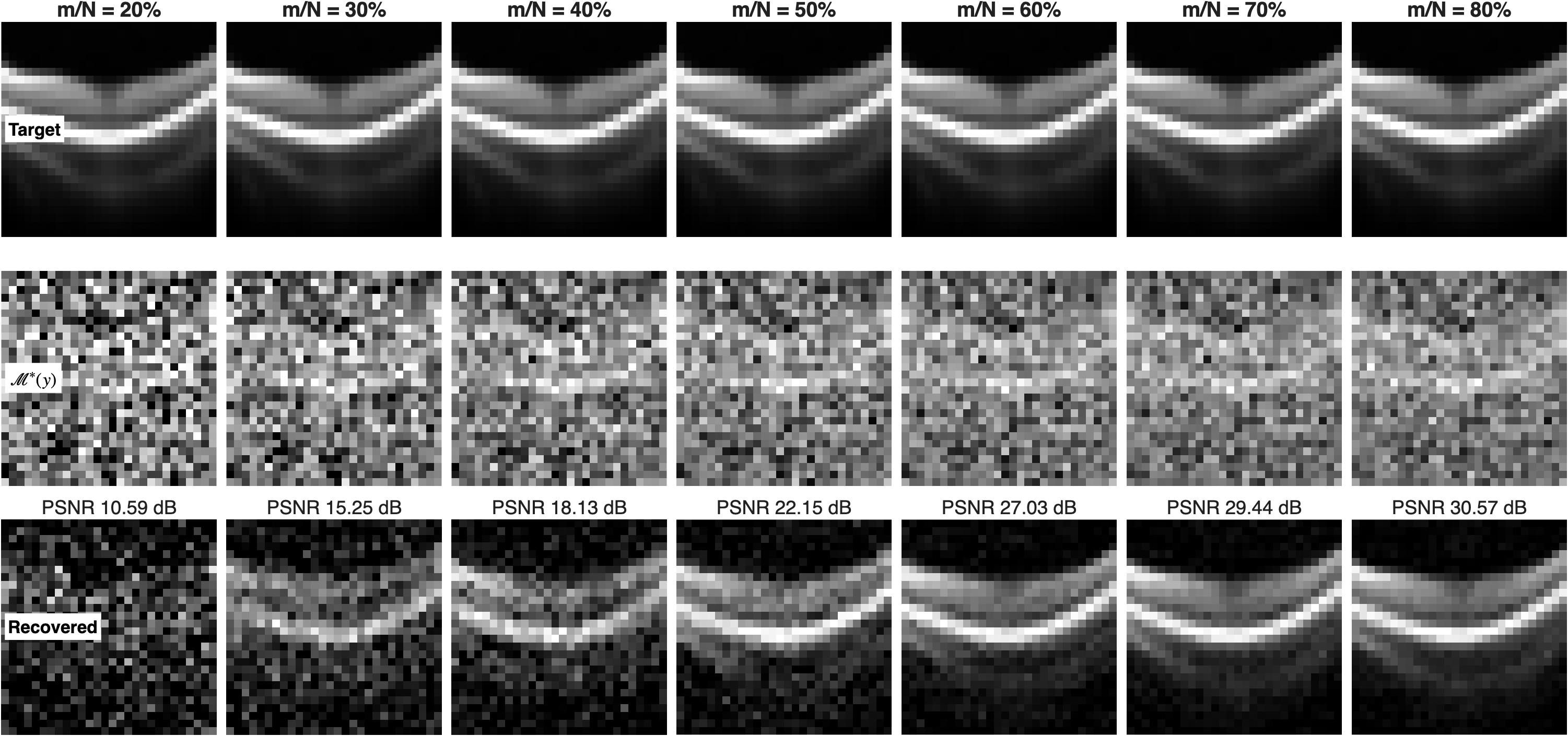}
    \caption{Rank-3 model mismatch for the untruncated preprocessed OCT
    tensor. The layout, representative-trial selection rule, and
    interpretation of the displayed PSNR values are the same as in
    Figure~\ref{fig:oct-exact}.
    Increasing $m/N$ restores progressively more retinal structure, while a
    nonzero discrepancy remains because the reconstruction has tubal rank
    at most 3.}
    \label{fig:oct-mismatch}
\end{figure}

Figure~\ref{fig:oct-summary} summarizes the relative recovery error
$\|\mathcal X_{\mathrm{rec}}-\mathcal{X}_{\mathrm{tar}}\|_F/\|\mathcal{X}_{\mathrm{tar}}\|_F$, PSNR, and SSIM averaged over the eight B-scans, using ten random starts at each measurement ratio. 
Here, $\mathcal X_{\mathrm{rec}}$ denotes the reconstructed tensor.
For the exact-rank experiment, a trial is counted as successful if the relative recovery error is at most $10^{-2}$; we do not report a success rate for the model-mismatch
experiment. The exact-rank target shows a sharp improvement: the median
relative recovery error decreases from $0.141$ at $m/N=0.5$ to $0.0174$,
$0.00525$, and $0.00328$ at $m/N=0.6$, $0.7$, and $0.8$, respectively.
The corresponding success rates are 0\%, 20\%, 60\%, and 100\%.
By contrast, the median error for the model-mismatch target decreases to
$0.0842$ at $m/N=0.8$, but remains above the rank-three approximation floor. These experiments use one OCT volume and a fixed Gaussian sensing
operator at each measurement ratio. The error bars in
Figure~\ref{fig:oct-summary} reflect variability across initializations.
Successful recovery from all ten random initializations at $m/N=0.8$
illustrates the empirical effectiveness of the balanced formulation
in this instance, beyond the small-t-RIP regime covered by
Theorem~\ref{thm:global-strict-saddle}.
Figures~\ref{fig:oct-exact} and~\ref{fig:oct-mismatch} complement
these aggregate statistics with representative individual
reconstructions for the exact-rank and model-mismatch targets,
respectively.

\section{Conclusion}
\label{sec:conclusion}

We studied the optimization landscape of balanced low-tubal-rank tensor sensing under the t-product. Under the t-RIP assumption, we proved that the objective has no spurious local minima and that every nonglobal critical point is a strict saddle, and established a quantitative strict-saddle characterization over the entire factor space. We further showed that the multi-rank determines the local geometry at the solution set. When every Fourier slice has rank equal to the tubal rank, the objective has quadratic growth transverse to the solution orbit. When some slice has smaller rank, the factorization remains overparameterized at that frequency even with the exact tubal rank, producing additional normal Hessian-kernel directions with quartic objective growth and cubic gradient growth. These directions explain the cubic-root proximity scale in the general landscape theorem.
Future work includes extending the landscape analysis to noisy measurements and approximate low-tubal-rank models.
It would also be interesting to investigate analogous geometric distinctions for other tensor settings.

\section*{Acknowledgment}
This work is in part supported by NSF grants ECCS-2409702, NSF Grant DMS-2603463, and OIA-2535317.

\section*{Declarations}
There is no conflict of interest. 

%
\bibliography{sn-bibliography}
\end{document}